\documentclass[a4paper]{article}

\usepackage[english]{babel}
\usepackage[utf8x]{inputenc}
\usepackage[T1]{fontenc}

\usepackage[a4paper,top=3cm,bottom=3cm,left=3cm,right=3cm,marginparwidth=1.75cm]{geometry}

\usepackage{amsmath,amssymb}
\usepackage{amsfonts} 
\usepackage{graphicx}
\usepackage{url}
\usepackage[colorinlistoftodos]{todonotes}
\usepackage[colorlinks=true, allcolors=blue]{hyperref}
\usepackage[ruled,vlined]{algorithm2e}

\usepackage{caption}
\usepackage{booktabs}       % professional-quality tables
\usepackage{nicefrac}       % compact symbols for 1/2, etc.
\usepackage{microtype}      % microtypography
\usepackage{xcolor}   

\usepackage{enumitem}
\usepackage{subfig}
\usepackage{makecell}
\usepackage{array}
\usepackage{threeparttable}

\newcommand{\Eall}{\Ecal_{all}}
\newcommand{\Eava}{\Ecal_{avail}}
\SetKwComment{Comment}{$\triangleright$}{}

\usepackage{mathtools}
\usepackage{amssymb}
\usepackage{latexsym}
\usepackage{dsfont}
\usepackage{mathrsfs}
\usepackage{amssymb}
\usepackage{amsfonts}
\usepackage{bm}
\usepackage{xspace}
\usepackage{amsthm}
\usepackage{multirow}
\newcommand{\R}{\mathbb{R}} % Real numbers
\newcommand*{\argmin}{\mathop{\mathrm{argmin}}}
\newcommand*{\argmax}{\mathop{\mathrm{argmax}}}

\newcommand{\wbf}{\mathbf{w}}

\newcommand{\Ecal}{\mathcal{E}}
\newcommand{\Fcal}{\mathcal{F}}
\newcommand{\Gcal}{\mathcal{G}}

\newcommand{\Ical}{\mathcal{I}}

\newcommand{\Lcal}{\mathcal{L}}
\newcommand{\Mcal}{\mathcal{M}}
\newcommand{\Ncal}{\mathcal{N}}

\newcommand{\Rcal}{\mathcal{R}}
\newcommand{\Scal}{\mathcal{S}}

\newcommand{\Vcal}{\mathcal{V}}

\newcommand{\Xcal}{\mathcal{X}}
\newcommand{\Ycal}{\mathcal{Y}}

\newcommand{\Ebb}{\mathbb{E}}

\newcommand{\Ibb}{\mathbb{I}}

\newcommand{\Pbb}{\mathbb{P}}

\newcommand{\Rbb}{\mathbb{R}}

\ifx\BlackBox\undefined
\newcommand{\BlackBox}{\rule{1.5ex}{1.5ex}}  % end of proof
\fi

\ifx\QED\undefined
\def\QED{~\rule[-1pt]{5pt}{5pt}\par\medskip}
\fi

\ifx\proof\undefined
\newenvironment{proof}{\par\noindent{\em Proof:\ }}{\hfill\BlackBox\\[.0mm]}
\fi

\ifx\theorem\undefined
\newtheorem{theorem}{Theorem}[section]
\fi

\ifx\example\undefined
\newtheorem{example}{Example}[section]
\fi

\ifx\property\undefined
\newtheorem{property}{Property}
\fi

\ifx\lemma\undefined
\newtheorem{lemma}{Lemma}[section]
\fi

\ifx\proposition\undefined
\newtheorem{proposition}{Proposition}[section]
\fi

\ifx\remark\undefined
\newtheorem{remark}{Remark}
\fi

\ifx\corollary\undefined
\newtheorem{corollary}{Corollary}[section]
\fi

\ifx\definition\undefined
\newtheorem{definition}{Definition}[section]
\fi

\ifx\conjecture\undefined
\newtheorem{conjecture}{Conjecture}[section]
\fi

\ifx\axiom\undefined
\newtheorem{axiom}[theorem]{Axiom}
\fi

\ifx\claim\undefined
\newtheorem{claim}[theorem]{Claim}
\fi

\ifx\assumption\undefined
\newtheorem{assumption}{Assumption}
\fi

\ifx\problem\undefined
\newtheorem{problem}{Problem}[section]
\fi

\ifx\fact\undefined
\newtheorem{fact}{Fact}
\fi

\newcommand{\rbr}[1]{\left(#1\right)}

\newcommand{\err}{\mathrm{err}}

\newcommand{\benr}{\begin{eqnarray}}
\newcommand{\eenr}{\end{eqnarray}}
\newcommand{\benrr}{\begin{eqnarray*}}
\newcommand{\eenrr}{\end{eqnarray*}}
\newcommand{\ben}{\begin{equation}}
\newcommand{\een}{\end{equation}}
\newcommand{\benn}{\begin{equation*}}
\newcommand{\eenn}{\end{equation*}}

\numberwithin{equation}{section}

\title{\textbf{Towards a Theoretical Framework of Out-of-Distribution Generalization}}
\author{Haotian Ye\footnote{Peking University; e-mail:\,\, \texttt{haotianye@pkue.edu.cn}},
Chuanlong Xie\footnote{Huawei Noah's Ark Lab; e-mail: \texttt{xie.chuanlong@huawei.com}}, 
Tianle Cai\footnote{Peking University; e-mail: \texttt{caitianle1998@pku.edu.cn}}, 
Ruichen Li\footnote{Peking University; e-mail: \texttt{xk-lrc@pku.edu.cn}},
Zhenguo Li\footnote{Huawei Noah's Ark Lab; e-mail: \texttt{Li.Zhenguo@huawei.com}}, 
Liwei Wang\footnote{Peking University; e-mail: \texttt{wanglw@cis.pku.edu.cn}}
}

\begin{document}
\maketitle

\begin{abstract}

Generalization to out-of-distribution (OOD) data is one of the central problems in modern machine learning. Recently, there is a surge of attempts to propose algorithms that mainly build upon the idea of extracting invariant features. Although intuitively reasonable, theoretical understanding of what kind of invariance can guarantee OOD generalization is still limited,
and generalization to arbitrary out-of-distribution is clearly impossible. In this work, we take the first step towards rigorous and quantitative definitions of 1) what is OOD; and 2) what does it mean by saying an OOD problem is learnable. We also introduce a new concept of expansion function, which characterizes to what extent the variance is amplified in the test domains over the training domains, and therefore give a quantitative meaning of invariant features. Based on these, we prove OOD generalization error bounds. It turns out that OOD generalization largely depends on the expansion function. As recently pointed out by \cite{gulrajani2020search}, any OOD learning algorithm without a model selection module is incomplete. Our theory naturally induces a model selection criterion. Extensive experiments on benchmark OOD datasets demonstrate that our model selection criterion has a significant advantage over baselines.

\end{abstract}

\tableofcontents

\section{Introduction}\label{sec:intro}

One of the most fundamental assumptions of classic supervised learning is the ``\textit{i.i.d.} assumption'', which states that the training and the test data are independent and identically distributed. However, this assumption can be easily violated in a reality \cite{Beery_2018_ECCV,ben2010theory,bengio2019meta,degrave2020ai,moreno2012unifying,recht2019imagenet,taori2020measuring} where the test data usually have a different distribution than the training data. This motivates the research on the out-of-distribution (OOD) generalization, or domain generalization problem, which assumes access only to data drawn from a set $\Eava$ of available domains during training, and the goal is to generalize to a larger domain set $\Eall$ including \emph{unseen} domains.

To generalize to OOD data, most existing algorithms attempt to learn features that are \emph{invariant} to a certain extent across training domains in the hope that such invariance also holds in unseen domains. For example, distributional matching-based methods \cite{ganin2016domain,li2018domain,sun2016deep} %\tianlermk{cite:DANN, CORAL, MMD, etc.} 
seek to learn features that have the same distribution across different domains; IRM~\cite{arjovsky2019invariant} and its variants \cite{ahuja2020invariant,koyama2020out,krueger2020out} learn feature representations such that the optimal linear classifier on top of the representation matches across domains. 

Though the idea of learning invariant features is intuitively reasonable, there is only limited theoretical understanding of what kind of invariance can guarantee OOD generalization. Clearly, generalization to an arbitrary out-of-distribution domain is impossible and in practice, the features can hardly be absolutely invariant from $\Eava$ to $\Eall$ unless all the domains are identical. So it is necessary to first formulate what OOD data can be generalized to, or, what is the relation between the available training domain set $\Eava$ and the entire domain set $\Eall$. Meanwhile, to what extent the invariance of features on $\Eava$ can be preserved in $\Eall$ should be rigorously characterized.

In this paper, we take the first step towards a general OOD framework by quantitatively formalizing the relationship between $\Eava$ and $\Eall$ in terms of the distributions of \emph{features} and provide OOD generalization guarantees based on our quantification of the difficulty of OOD generalization problem. Specifically, we first rigorously formulate the intuition of invariant features used in previous works by introducing the ``variation'' and ``informativeness'' (Definition~\ref{def_invariance} and \ref{def_informativeness}) of each feature. Our theoretical insight can then be informally stated as: for learnable OOD problems, if a feature is informative for the classification task as well as invariant over $\Eava$, then it is still invariant over $\Eall$. In other words, invariance of informative features in $\Eava$ can be preserved in $\Eall$. We further introduce a class of functions, dubbed expansion function (Definition~\ref{def_expan}), to quantitatively characterize to what extent the variance of features on $\Eava$ is amplified on $\Eall$. 

Based on our formulation, we derive theoretical guarantees on the OOD generalization error, i.e., the gap of largest error between the domain in $\Eava$ and domain in $\Eall$. Specifically, we prove the upper and lower bound of OOD generalization error in terms of the expansion function and the variation of learned features over $\Eava$. 
Our results theoretically confirm that 1) the expansion function can reflect the difficulty of OOD generalization problem, i.e., problems with more rapidly increasing expansion functions are harder and have worse generalization guarantees; 2) the generalization error gap can tend to zero when the variation of learned features tend to zero, so minimizing the variation in $\Eava$ can reduce the generalization error.

As pointed out by \cite{gulrajani2020search}, any OOD algorithm without a specified model selection criterion is not complete. Since $\Eall$ is unseen, hyper-parameters can only be chosen according to $\Eava$. Previous selection methods mainly focus on validation accuracy over $\Eava$, which is only a biased metric of OOD performance. On the contrary, a promising model selection method should instead be predictive of OOD performance. Inspired by our bounds, we propose a model selection method to select models with high validation accuracy and low variation, which corresponds to the upper bound of OOD error. The introduction of a model's variation relieves the problem of classic selection methods, in which models that overfit $\Eava$ tend to be selected. Experimental results show that our method can outperform baselines and select models with higher OOD accuracy.

\paragraph{Contributions.} We summarize our major contributions here:
\begin{itemize}
    \item We introduce a quantitative and rigorous formulation of OOD generalization problem that characterizes the relation of invariance over the training domain set $\Eava$ and test domain set $\Eall$. The core quantity in our characterization, the expansion function, determines the difficulty of an OOD generalization problem.
    \item We prove novel OOD generalization error bounds based on our formulation. The upper and lower bounds together indicate that the expansion function well characterizes the OOD generalization ability of features with different levels of variation.
    \item We design a model selection criterion that is inspired by our generalization bounds. Our criterion takes both the performance on training domains and the variation of models into consideration and is predictive of OOD performance according to our bounds. Experimental results demonstrate our selection criterion can choose models with higher OOD accuracy.
\end{itemize}

The rest of the paper is organized as follows: Section~\ref{sec_preli} is our preliminary. 
In Section~\ref{sec_formalization}, we give our theoretical formulation.
%and real-world example. 
Section~\ref{section_bound} gives our generalization bound. 
We propose our model selection method in Section~\ref{sec_selection}.
In Section~\ref{sec_experiment} we conduct experiments on expansion function and model selection.
We review more related works in Section~\ref{related_work} and conclude our work in Section~\ref{sec_conclusion}.

\section{Preliminary}\label{sec_preli}

Throughout the paper, we consider a multi-class classification task $\Xcal \to \Ycal = \{1, \ldots,K\}$.\footnote{Note that our framework can be generalized to other kinds of problems easily.} 
Let $\Eall$ be the domain set we want to generalize to, and
$\Eava \subseteq \Eall$ be the available domain set, i.e., all domains we have during the training procedure. 
We denote $\rbr{X^e, Y^e}$ to be the input-label pair drawn from the data distribution of domain $e$.
The OOD generalization goal is to find a classifier $f^*$ that minimizes the \emph{worst-domain} loss on $\Eall$:
\ben\label{eq31}
f^* = \argmin_{f \in \Fcal} \Lcal (\Eall,f),\ 
\Lcal(\Ecal,f) \triangleq \max_{e \in \Ecal} \Ebb\big[\ell \big(f(X^e), Y^e\big) \big]
\een
where $\Fcal : \Xcal \to \Rbb^K$ is the the hypothetical space and $\ell(\cdot,\cdot)$ is a loss function. Similar to previous works~\cite{arjovsky2019invariant,creager2021exchanging,jin2021enforcing,krueger2020out}, we assume that $f$ can be decomposed into $g \circ  h$, where $g \in \Gcal: \Rbb^d \to \R^K$ is the top classifier and $h :\Xcal \to \Rbb^d $ is a $d$-dimensional feature extractor, i.e., 
\benrr
h(x) = (\phi_1(x), \phi_2(x), \ldots, \phi_d(x))^\top, \quad \phi_i \in \Phi.
\eenrr
Here $\Phi$ is the set of scalar feature maps which map $\Xcal$ to $\Rbb$ and $d$ is fixed. We will call each $\phi\in\Phi$ a feature for simplicity.
Given a domain $e$, we denote the $d$-dimensional random vector $h(X^e)$ as $h^e$, one-dimensional feature $\phi(X^e)$ as $\phi^e$, and the conditional distribution of $h^e,\phi^e$ given $Y^e=y$ as $\Pbb(h^e|y),\Pbb(\phi^e|y).$
For simplicity, we assume the data distribution is balanced in every domain, i.e., $P(Y^e = y) = \frac 1 K,\forall y\in\Ycal,e\in\Eall$.
Our framework can be easily extended to the case where the balanced assumption is removed, with an additional term corresponding to the imbalance adding to the generalization bounds.

\section{Framework of OOD Generalization Problem}\label{sec_formalization}

The main challenge of formalizing the OOD generalization problem is to mathematically describe the connection between $\Eava$ and $\Eall$ and how generalization depends on this relation. 
Towards this goal, we introduce several quantities to characterize the relation of \emph{feature distributions} over different domains and bridge $\Eava$ and $\Eall$ by expansion function (Definition~\ref{def_expan}) over the quantities we have introduced.
Our framework is motivated by the understanding that, in an OOD generalization task, certain ``property'' of ``good'' features in $\Eava$ should be ``preserved'' in $\Eall$ (the reason is described in Section~\ref{sec:intro}). 
In Section~\ref{sec:detailed_formalization}, we will go into details on what we mean by ``property'' (variation, Definition~\ref{def_invariance}), ``good'' (informativeness, Definition~\ref{def_informativeness}), and ``preserved'' (measured by expansion function).
In Section~\ref{sec_assumption_example}, we further illustrate the key concepts in our framework by a real-world OOD problem.

\subsection{Formalizing OOD Problem by Quantifying Feature Distribution}
\label{sec:detailed_formalization}

We first introduce the concepts ``variation" and ``informativeness" of a feature $\phi$. 
The first one is what we expect to be preserved in $\Eall$ and the second one characterizes what features will be considered.
Specifically, let $\rho(\mathbb P,\mathbb Q)$ be a symmetric ``distance'' of two distributions. Note that $\rho$ can have many choices, like $L_2$ Distance, Total Variation and symmetric KL-divergence, etc.
The variation and informativeness are defined as follows:
\begin{definition}[Variation]\label{def_invariance}
The variation of feature $\phi(\cdot)$ across a domain set $\Ecal$ is
\begin{align}
\Vcal_{\rho}(\phi, \Ecal) = \max_{y\in\Ycal} \sup_{e, e'\in\Ecal} \rho \big(\Pbb(\phi^e|y), \Pbb(\phi^{e'}|y) \big).
\end{align}
A feature $\phi(\cdot)$ is $\varepsilon$-invariant across $\Ecal$, if $\varepsilon \geq \Vcal(\phi, \Ecal)$ (We omit the subscript $\rho$ in case of no ambiguity).
\end{definition}
\begin{definition}[Informativeness]\label{def_informativeness}
The informativeness of feature $\phi(\cdot)$ across a domain set $\Ecal$ is
\benr
\Ical_\rho(\phi, \Ecal) = 
\frac{1}{K(K-1)}\sum_{\substack{y\neq y'\\ y,y'\in \Ycal}}
\min_{e\in\Ecal} \rho \big( \Pbb(\phi^e|y) , \Pbb(\phi^e|y')\big).
\eenr
A feature $\phi(\cdot)$ is $\delta$-informative across $\Ecal$, if $\delta \leq \Ical(\phi, \Ecal).$
\end{definition} 

The variation $\Vcal(\phi, \Ecal)$ measures the stability of $\phi(\cdot)$ over the domains in $\Ecal$ and the informativeness $\Ical(\phi, \Ecal)$ captures the ability of $\phi(\cdot)$ to distinguish different labels.
We would like to highlight that the variation and informativeness are defined on each one-dimensional feature $\phi(\cdot)$.
Unlike previous distance between distributions defined in $d$-dimensional space, our definitions are more reasonable and practical, since it can be easily calculated and analyzed.

We are now ready to introduce the core quantity for connecting $\Eava$ and $\Eall$. Our motivation, as elaborated in the introduction section, is that, if a feature is informative for the classification task and invariant over $\Eava$,
then to enable OOD generalization from $\Eava$ to $\Eall$, it should be still invariant over $\Eall$. So the relation between $\Vcal(\phi, \Eava)$ and $\Vcal(\phi, \Eall)$ of an informative feature captures the feasibility and difficulty of OOD generalization. To quantitatively measure this relation, we define the following function class:

\begin{definition}[Expansion Function] \label{def_expan}
We say a function $s:\mathbb R^+ \cup \{0\} \to \mathbb R^+ \cup \{0, +\infty\}$ is an expansion function, iff the following properties hold: 
1) $s(\cdot)$ is monotonically increasing and $s(x)\geq x,\forall x\geq0$; 2) $\lim_{x\to 0^+} s(x) = s(0) = 0$.
\end{definition}

This function class gives a full characterization of how the variation between $\Eava$ and $\Eall$ is related. 
Based on this function class, we can introduce our formulation of the learnability of OOD generalization as follows:

\begin{definition}[Learnability]\label{def_learn}
Let $\Phi$ be the feature space and $\rho$ be a distribution distance. We say an OOD generalization problem from $\Ecal_{avail}$ to $\Ecal_{all}$ is \emph{learnable} if there exists an expansion function $s(\cdot)$ and $\delta\ge 0$, such that: for all $\phi\in \Phi$ satisfying $\Ical_\rho(\phi,\Eava) \geq \delta$, we have $s(\Vcal_\rho(\phi, \Ecal_{avail}))\geq \Vcal_\rho(\phi, \Ecal_{all})$. If such $s(\cdot)$ and $\delta$ exist, we further call this problem $(s(\cdot), \delta)$-learnable. If an OOD generalization problem is not learnable, we call it unlearnable.
\end{definition}
To understand the intuition and rationality of our formulation, several discussions are in order.
\paragraph{Properties of the expansion function.} 
In Definition~\ref{def_expan}, we highlight two properties of the expansion function. The first property comes naturally from the monotonicity properties of variation: any $\varepsilon_1$-invariant feature is also $\varepsilon_2$-invariant for $\varepsilon_2 \geq \varepsilon_1$; and  $\Vcal(\phi, \Ecal_1) \leq \Vcal(\phi, \Ecal_2)$ for any $\Ecal_1 \subseteq \Ecal_2$. The monotonicity also implies that larger $\Eall$ will induce larger $s(\cdot)$\footnote{When we talk about the scale of $s(\cdot)$, e.g. it is larger / smaller, we mean the comparison of two expansion function, rather than the comparison along a function. 
} 
and it is also harder to be generalized to. From this view, we can see that the scale of $s(\cdot)$ can reflect the difficulty of OOD generalization.
The second property is more crucial since it formulates the intuition that if an informative feature is \emph{almost} invariant in $\Eava$, it should remain invariant in $\Eall$. 
Without this assumption, OOD generalization can never be guaranteed because we cannot predict whether an invariant and informative feature in $\Eava$ will vary severely in \emph{unseen} $\Eall$.

\paragraph{Necessity of informativeness.} 
We include a seemingly redundant quantity informativeness in the definition of learnability. However, this term is necessary because only informative features are responsible for the performance of classification. Non-informative but invariant features over $\Eava$ may only capture some noise that is irrelevant to the classification problem, and we shall not expect the noise to be invariant over $\Eall$. Moreover, we show in Figure~\ref{fig_OH_expan_func} that in practice, 
many invariant but useless features in $\Eava$ vary a lot in $\Eall$, and adding the constraint of informativeness makes the expansion function reasonable.
In addition, there are multiple choices of $(s(\cdot),\delta)$ to make an OOD generalization problem learnable: larger $\delta$ will filter out more features, and so $s(\cdot)$ can be smaller (flatter).
This multiplicity will result in a trade-off between $s(\cdot)$ and $\delta$, which will be discussed in Section~\ref{sec_assumption_example}.

\paragraph{Two extreme cases: \textit{i.i.d.} \& unlearnable.} 
To better understand the concept of learnability, we consider two extreme cases. 
(1) The first example is when all data from different $e\in\Eall$ are \emph{identically} distributed, i.e., the classic supervised learning setting. This problem is $(s(\cdot),0)$-learnable with $s(x)=x$, implying no extra difficulty in OOD generalization. 
(2) As an example of unlearnable, consider the following case (modified from Colored MNIST \cite{arjovsky2019invariant}):   
For $e \in \Eava$, images with label $0$ always has a red background while images with label $1$ has a blue background. 
For $e \in \Eall \setminus \Eava$, this relationship is entirely inverse.
Since data from different $e \in \Eava$ are identically distributed but different from other $e\in\Eall$, no expansion function can make it learnable, i.e., it is OOD-unlearnable.
The unlearnability of this case also coincides with our intuition: 
Without prior knowledge, it is not clear from merely the training data, whether the task is to distinguish  digit 0 from 1, or to distinguish color red from blue.
As a result, generalization to $\Eall$ cannot be guaranteed.

\section{Generalization Bound}\label{section_bound}

In this section, we consider an OOD generalization problem from $\Eava$ to $\Eall$,
and our goal is to analyze the OOD generalization error of classifier $f=g\circ h$ defined by
\benrr
\err(f) = \Lcal(\Eall,f) - \Lcal(\Eava,f),
\eenrr
where we assume the loss function $l(\cdot,\cdot)$ is bounded by $[0,C].$
We prove two upper bounds (\ref{general ood bound}, \ref{bound_linearg}) as well as a lower bound (\ref{lower bound}) for $\text{err}(f)$ based on our formulation. 
Our bounds together provide a complete characterization of the difficulty of OOD generalization.
Since we expect that an invariant classifier can generalize to unseen domains, we hope to bound $\text{err}(f)$ in terms of the certain variation of $f$.
To this end, we define the variation and informativeness of $f$ in terms of its features, i.e.,
\benrr
\Vcal^{\text{sup}}(h, \Eava) &\triangleq &\sup_{\beta \in \mathcal S^{d-1}} \Vcal(\beta^\top h, \Eava),\\
\Ical^{\text{inf}}(h, \Eava) &\triangleq &\inf_{\beta \in \mathcal S^{d-1}} \Ical(\beta^\top h, 
\Eava),
\eenrr
where $(\beta^\top h)(x) = \beta^\top h(x)$ is a feature and $\Scal^{d-1} = \{\beta \in \mathbb R^d: \|\beta\|_2 = 1\}$ is the unit $(d-1)$-sphere.

\paragraph{Necessity of using supremum over linear combination.} One seemingly plausible definition of the variation of a classifier $f$ can be the supremum over all $\Vcal(\phi_i,\Eava),i\in[d]$. 
However, as is shown in Appendix~1, it is possible that two high dimensional joint distributions have close marginal distribution in each dimension, while they do not overlap.
In other words, there exist cases where $\Vcal(\phi_i,\Eall) = 0,\forall i\in[d]$ but after applying the top model $g$ over $\phi_i$'s, the distribution varies a lot in $\Eava$.
Our definition comes from the simple idea that the class of top model $\mathcal G$ is at least a linear space, so we should at least consider the variation of every (normalized) linear combination of $h(\cdot)$.
With this, we can guarantee the joint distribution distance is still small. %We assume the decay upper bound on the OOD generalization error, which general models stated as follows.

\begin{theorem}[Main Theorem]\label{general ood bound}
Suppose we have learned a classifier $f(x) = g(h(x))$ such that $\forall e \in \Eall$ and $\forall y \in \Ycal$, $p_{h^e|Y^e} (h|y) \in L^2(\Rbb^d).$
%and $\hat p_{h^e|Y^e} (t|y) \in L^1(\Rbb^d)$
Denote the characteristic function of random variable $h^e|Y^e$ as $\hat p_{h^e|Y^e}(t|y) = \mathbb E [\exp\{i \langle t,h^e \rangle\}|Y^e=y].$ 
Assume the hypothetical space $\mathcal F$ satisfies the following regularity conditions that $\exists \alpha,M_1,M_2 >0, \forall f \in \Fcal, \forall e \in \Eall, y\in\Ycal$,
    \begin{align}\label{concentration_assumption}
    \int_{h\in\R^d} p_{h^e|Y^e}(h|y) |h|^\alpha \mathrm d h\leq M_1 \quad \text{and} \quad \int_{t\in\R^d} |\hat p_{h^e|Y^e}(t|y)| |t|^\alpha \mathrm dt \leq M_2.
    \end{align}
If $(\Eava,\Eall)$ is $\big(s(\cdot),\Ical^{\text{inf}}(h,\Eava)\big)$-learnable under $\Phi$ with Total Variation $\rho$\footnote{For two distribution $\mathbb P,\mathbb Q$ with probability density function $p,q$, $\rho(\mathbb P,\mathbb Q) = \frac12\int_x |p(x) - q(x)|\mathrm dx$.}, then we have
\benr\label{main}
\err(f) \leq O\Big(s\big(\Vcal^{\text{sup}}_\rho (h,\Eava)\big)^{\frac{\alpha^2}{(\alpha+d)^2}}\Big).
\eenr
Here $\rho$ is total variation distance, and $O(\cdot)$ depends on $d,C,\alpha,M_1,M_2$.
\end{theorem}

The above theorem holds for a general classifier learned by any algorithms. Due to its generality, we need to introduce some technical regularity conditions on the density function. 
The assumption (\ref{concentration_assumption}) assume the decay rate of density and its characteristic function, which is common in the literature, e.g. \cite{cavalier2000efficient}.
This theorem demonstrates that, the generalization error can be bounded by a function of the variation of $h$, and it converges to $0$ as the variation approaches to $0$.
Under some special but typical case where the top model $g$ is linear, we can further show that even without the regularity conditions in Theorem~\ref{general ood bound}, we have a much better (linear) convergence rate.

\begin{theorem}[Linear Top Model]\label{bound_linearg}
Consider any loss satisfying $\ell(\hat y,y) = \sum_{k=1}^K \ell_0(\hat y_k , y_k).$\footnote{This decomposition is a technical assumption to make the proof more convenient. Truncated square loss or Truncated absolute loss satisfy this assumption.}
For any classifier with linear top model $g$, i.e., 
\benrr
f(x) = A h(x) + b \quad \text{with} \quad  A\in \Rbb^{K\times d},\,\, b \in \Rbb ^{K},
\eenrr
if $(\Eava,\Eall)$ is $\big(s(\cdot),\Ical^{\text{inf}}(h,\Eava)\big)$-learnable under $\Phi$ with Total Variation $\rho$, then we have
\benr
\err(f) \leq  O\Big(s\big(\Vcal^{\text{sup}}(h,\Eava)\big)\Big).
\eenr
Here $O(\cdot)$ depends only on $d$ and $C$.
\end{theorem}

\paragraph{Discussion.} Theorem~\ref{general ood bound} shows that, for any model, the generalization gap depends largely on the model's variation captured by $\Vcal^{\text{sup}} (h,\Eava)$.
The result is irrelevant to the algorithm and provides a guarantee for the generalization gap from $\Eava$ to $\Eall$, so long as the learned model $f$ is invariant, i.e. $\Vcal^{\text{sup}} (h,\Eava)$ is small.
When $s(\cdot)$ is fixed, a model with smaller $\Vcal^{\text{sup}}(h,\Eava)$ results in a smaller gap, which matches our understanding that invariant features in $\Eava$ are somehow invariant in $\Eall$.
When $\Vcal^{\text{sup}} (h,\Eava)$ is fixed, more difficult generalization will generate a larger expansion function, which leads to a larger gap.
For the Gaussian class with bounded mean and variance, $\alpha \gg d$ and the convergent rate is almost linear. 

However, without any constraint to $g$, the convergent rate might be small.
Theorem~\ref{bound_linearg} then offers a generalization bound with a \emph{linear} convergent rate under mild assumptions when $g$ is linear, which is common in reality.
It relaxes the concentration assumption (Formula~\ref{concentration_assumption}) and asks only for the integrability of the density.
The convergent rate is identical to the convergent rate of the expansion function, showing that $s(\cdot)$ captures the generalization quite well.

\paragraph{Proof Sketch of Theorem~\ref{general ood bound}.} 
The proof of the main result, Theorem~\ref{general ood bound}, is decomposed into the following steps. 
First, we transform $\text{err}(f)$ into the Total Variation of joint distributions of features in different domains (Step 1).
To bound the Total Variation, it is sufficient to bound the distance of the corresponding Fourier transform, and further, it is equivalent to bound the Radon transform of joint distributions (Step 2).
Eventually, we show that $\Vcal^{\text{sup}}(\beta^\top h ,\Eava)$ can be used to bound the Radon transform, which finishes the proof (Step 3).

{\bf Step~1}. The OOD generalization error can be bounded as:
\benr\label{sketch_formula1}
\text{err}(f) \leq \sup_{(e,e')\in(\Eava,\Eall)}\frac C K \sum_{y\in \Ycal} \int_{h \in \Rbb^d} \big|p_{h^e|Y^e}(h|y)-p_{h^{e'}|Y^{e'}}(h|y))\big|\mathrm dh.
\eenr

{\bf Step~2}. According to the assumption (\ref{concentration_assumption}), the dominant term in (\ref{sketch_formula1}) is
\benr \label{sketch_formula2}
\int_{|h| \leq r_1}\Big| \int_{|t|\leq r_2} e^{-i \langle h,t \rangle}\big(\hat p_{h^e|Y^e}(t|y)-\hat p_{h^{e'}|Y^{e'}}(t|y))\big) \mathrm dt \Big| \mathrm d t,
\eenr
where $r_1$ and $r_2$ are well-selected scalars that depend on $s\big(\Vcal^{\text{sup}}_\rho (h,\Eava)\big).$ By the Projection Theorem \cite{korostelev2012minimax,natterer2001mathematics} and the Fourier Inversion Formula, (\ref{sketch_formula2}) is bounded above by
\benr\label{sketch_formula3}
O(r_1^d r_2^d) \times \int_{u \in \R} \big|\Rcal_{e'}(\beta, u) - \Rcal_e(\beta, u) \big| \mathrm d u,
\eenr
where $\Rcal_e(\beta, u)$ is the Radon transform of $p_{h^{e}|Y^{e}}(t|y).$

{\bf Step~3}. The right-hand side of Formula~\ref{sketch_formula2} can be bounded by $O\big(r_1^d r_2^d s\big(\Vcal^{\text{sup}}_\rho (h,\Eava)\big)\big).$ 
We finish the proof by selecting appropriate $r_1$ and $r_2$ to balance the rate of the dominant term and other minor terms.
For more details, please see Appendix~2 for the complete proofs.

Now we turn to the lower bound of $\text{err}(f)$.

\begin{theorem}[Lower Bound]\label{lower bound}
Consider $0$-$1$ loss: $\ell(\hat y, y) = \mathbb I(\hat y \neq y)$. 
For any $\delta > 0$ and any expansion function satisfying 1) $s'_+(0) \triangleq \lim_{x\to0^+ } \frac {s(x) - s(0)}{x} \in (1,+\infty)$; 2) exists $k>1, t >0$, s.t. $kx \leq s(x) < +\infty, x \in [0,t]$,
%For any linear expansion function $s(x) = kx,k\in(1,+\infty)$ 
there exists a constant $C_0$ and an OOD generalization problem $(\Eava,\Eall)$ that is $(s(\cdot),\delta)$-learnable under linear feature space $\Phi$ w.r.t symmetric KL-divergence $\rho$, s.t. $\forall \varepsilon\in [0,\frac t 2]$, the optimal classifier $f$ satisfying $\Vcal^{\text{sup}}(h,\Eava) = \varepsilon$ will have the OOD generalization error lower bounded by 
\benr
\err(f)\geq C_0\cdot s(\Vcal^{\text{sup}}(h,\Eava)).
\eenr
\end{theorem}

Theorem~\ref{lower bound} shows that $\text{err}(f)$ of optimal classifier $f$ is lower bounded by its variation.
Here ``optimal'' means the classifier that minimize $\mathcal L(f,\Eava)$.
Altogether, the above three theorems offer a bidirectional control of OOD generalization error, showing that our formulation can  offer a fine-grained description of most OOD generalization problem in a theoretical way. 
To pursue a good OOD performance, OOD algorithm should focus on improving predictive performance on $\Eava$ and controlling the variation $\Vcal^{\text{sup}} (h,\Eava)$ simultaneously.
Note that this bound starts from population error, and we call for future works to combine our generalization bound and traditional bound from data samples to population error, giving a more complete characterization of the problem.

 \section{Variation as a Factor of Model Selection Criterion }\label{sec_selection}

As is pointed out in \cite{gulrajani2020search}, model selection has a significant effect on domain generalization, and any OOD algorithm without a model selection criterion is not complete. 
\cite{gulrajani2020search} trained more than 45,900 models with different algorithms, and results show that when traditional selection methods are applied, none of OOD algorithms can outperform ERM \cite{vapnik1992principles} by a significant margin. This result is not strange, since traditional selection methods focus mainly on (validation) accuracy, which is biased in OOD generalization \cite{gulrajani2020search,ye2021out}.
A very typical example is Colored MNIST \cite{arjovsky2019invariant}, where the image is colored according to the label, but the relationship varies across domains.
As explained in \cite{arjovsky2019invariant}, ERM principle will only capture this spurious feature (color) and performs badly in $\Eall$. 
Since ERM is exactly minimizing loss in $\Eava$, any model selection method using validation accuracy alone is likely to choose ERM rather than any other OOD algorithm \cite{ye2021out}.  Thus no algorithm will have a significant improvement compared to ERM.

A natural question arises: what else can we use, in addition to accuracy?
Theorem~\ref{general ood bound} points out that, learning feature with small variation across $\Eava$ is important for decreasing OOD generalization error.
Once a model $f$ achieves a small $\Vcal^{\text{sup}}(h,\Eava)$, then $\text{err}(f)$ will be small. If the validation accuracy is also high, we shall know that the OOD accuracy will remain high.
To this end, we propose our heuristic selection criterion (Algorithm~\ref{Model Selection}).
Instead of considering validation accuracy alone, we combine it with feature variation and \emph{select the model with high validation accuracy as well as low variation.}

\begin{algorithm}
\SetAlgoLined
\textbf{Input:} available dataset $\Xcal_{avail} = (\Xcal_{train},\Xcal_{val})$, candidate models set $\Mcal$, var\_acc\_rate $r_0$. \\
%$\hat \Mcal = \{\}$\\
\For{$f=g\circ h$ in $\Mcal$}{
    %$\text{whether\_info},\hat \Vcal_f \leftarrow $ Algorithm~\ref{variation Criterion}($\Xcal_{train},h,\delta_0$)\Comment*[r]{Calculate model's \xcl{variation}}
    \For{$i$ in $[d]$}
    {
        $\hat \Vcal_i \leftarrow \max_{y \in \Ycal,\Xcal^e\neq\Xcal^{e'} \in \Xcal_{avail}}\text{Total Variation}(\Pbb(\phi^e_i|y),\Pbb(\phi_i^{e'}|y))$\Comment*[r]{Use GPU KDE}
        %$\hat \Ical_i \leftarrow \text{mean}_{y\neq y'\in\Ycal}\min_{\Xcal^e \in \Xcal_{avail} }\text{Total Variation}(\Pbb(\phi^e_i|y),\Pbb(\phi_i^{e}|y')$\Comment*[r]{Use GPU KDE}
    }
    $ \Vcal_f \leftarrow \text{mean}_{i\in[d]} \hat \Vcal_i$\\
    %$\hat \Vcal_f \leftarrow $ Algorithm~\ref{variation Criterion}($\Xcal_{train},h$)\Comment*[r]{Calculate model's \xcl{variation}}
    %\If{$\text{whether\_info}$ is True}
    %{
        %$\hat \Mcal \leftarrow \hat \Mcal  \cup \{f\}$\Comment*[r]{filter out non-informative model}
    $\text{Acc}_f\leftarrow$ compute validation accuracy of $f$ using $\Xcal_{val}$
    %}
}
%$\alpha \leftarrow 1/\text{std}_{  \Mcal}(\text{Acc}_f), \beta \leftarrow 1/\text{std}_{ \Mcal}(\hat \Vcal_f)$ \Comment*[r]{calculate standard deviation} 
\textbf{Return} $\argmax_{f \in  \Mcal}( \text{Acc}_f - r_0 \Vcal_f)$
\caption{Model Selection}\label{Model Selection}
\end{algorithm}

We briefly explain Algorithm~\ref{Model Selection} here.
For each candidate model, we calculate its variation using the average of each feature's variation, i.e., $\frac{1}{d}\sum_{i \in [d]} \Vcal(\phi_i,\Xcal_{avail})$. 
When deriving the bounds, we use $\Vcal^{\text{sup}}$ instead of their average because we need to consider \emph{the worst case}, i.e., the worst top model.
In practice, we find out that the average of $\Vcal(\phi_i,\Xcal_{avail})$ is enough to improve selection.

Our criterion of model selection is
\benr \label{formula_criterion}
\text{Acc}_f - r_0 \Vcal_f,
\eenr
i.e., we select a model with high validation accuracy and low variation \emph{simultaneously}.
Here $r_0$ is a hyper-parameter representing the concrete relationship between $\text{err}(f)$ and $\Vcal_f$.
Although we have already used one hyper-parameter to help select multiple hyper-parameter combinations, it is natural to ask whether we can further get rid of the selection of $r_0$.
Since $r_0$ represents the relationship between variation and accuracy, which is actually determined by the unknown expansion function, explicitly calculating $r_0$ is not possible.
However, we can empirically estimate $r_0$ using
$
r_0 = \frac{\text{Std}_{f \in \hat \Mcal} \text{Acc}_f}{\text{Std}_{f\in \hat \Mcal} \Vcal_f},
$
where $ \hat \Mcal \subset \Mcal$ is the model with not bad validation accuracy.
We do not use the whole set $ \Mcal$ because some OOD algorithms will perform extremely bad when the penalty is huge, and these models will influence our estimation of the ratio.
Since high validation means large informativeness in learned features,  the use of $\hat \Mcal$ is an implicit application of informative assumption.

As shown in Section~\ref{experimental_result}, our method can select models with higher OOD accuracy in various OOD datasets.
We also explain in Appendix~3 why our method can outperform the traditional method in Color MNIST, where the dataset is hand-make and simple enough to calculate the expansion function.

\section{Experiments}\label{sec_experiment}

In this section, we conduct experiments to compare our model selection criterion (Section~\ref{sec_selection}) with the baseline method\footnote{Our experiments is conducted in DomainBed: \url{https://github.com/facebookresearch/DomainBed}.} \cite{gulrajani2020search}.
Since both the variation and informativeness in Definition~\ref{def_invariance} are based on one-dimensional features, we can directly estimate these quantities feature-by-feature and design model selection method based on them. 
To verify the existence of the expansion function and to see what it's like in a real-world dataset, we plot nearly 2 million features trained in a common-used OOD dataset and compute their variation and informativeness. We then draw the expansion function for this problem.

\subsection{Experiments on Model Selection}\label{experimental_result}

In this section, we conduct experiments to compare the performance of models selected by our method and by validation accuracy. 
We train models on different datasets, different $\Eava$, and select models according to a different selection criteria. 
We then compare the OOD accuracy of selected models.

\paragraph{Settings}
We train our model on three benchmark OOD datasets (PACS \cite{li2017deeper}, OfficeHome \cite{venkateswara2017deep}, VLCS  \cite{torralba2011unbiased}) and consider all possible selections of $(\Eava,\Eall)$ .
We choose ResNet--50 as our network architecture.
We use ERM \cite{vapnik1992principles} and four {common-used} OOD algorithms (CORAL \cite{sun2016deep}, Inter-domain Mixup \cite{yan2020improve}, Group DRO \cite{sagawa2019distributionally}, and IRM \cite{arjovsky2019invariant}).
For each environment setup, we train 200 models using different algorithms, penalties, learning rates, and epoch. 
After training, we employ different selection methods and compare the OOD accuracy of the selected models. 
As stated in Section \ref{sec_selection}, we use the standard deviation of $\Vcal$ and validation accuracy in $\hat \Mcal$ to estimate $r_0$, where $\hat \Mcal = \{f\in\Mcal: \text{Acc}_f \geq \max_{\hat f} \text{Acc}_{\hat f} - 0.1 \}.$
{Note that calculating $\Vcal(\phi_i,\mathcal X_{avail})$ takes calculus many times,} so we design a parallel GPU kernel density estimation to speed up the whole process a hundred times and manage to finish one model in seconds. For more details about the experiments, see Appendix~4.

 \begin{table}[htp]
    \centering

    \caption{Model Selection Result. ``Env'' denotes the unseen domain during training. ``Val'' denotes the OOD accuracy of model selected by validation accuracy. }
        \vspace{10pt}
    \begin{tabular}{|c|c|c c c c|c|c|}
         \hline
        \multirow{3}*{PACS} & Env & A & C & P & S&avg&acc inc\\
        \cline{2-8}
        &Val&85.20\%&80.42\%&96.17\%&77.86\%&84.91\%&-\\
        \cline{2-8}
        &Ours&\textbf{88.72\%}&\textbf{81.74\%}&\textbf{96.83\%}&\textbf{79.00\%}&\textbf{86.57}\%&1.66\%$\uparrow$\\
        \hline
        \multirow{3}*{OfficeHome} & Env & A & C & P & R&avg&acc inc\\
        \cline{2-8}
        &Val&61.85\%&\textbf{55.56\%}&74.72\%&76.25\%&67.09\%&-\\
        \cline{2-8}
        &Ours&\textbf{65.76\%}&55.07\%&\textbf{75.20\%}&\textbf{76.31\%}&\textbf{68.09\%}&1.00\%$\uparrow$\\
        \hline
        \multirow{3}*{VLCS} & Env & C & L & S & V&avg&acc inc\\
        \cline{2-8}
        &Val&97.46\%&64.83\%&\textbf{69.50\%}\footnotemark[6]&\textbf{70.97\%}&75.69\%&-\\
        \cline{2-8}
        &Ours&\textbf{97.81\%}&\textbf{66.98\%}&\textbf{69.50\%}&\textbf{70.97\%}&\textbf{76.32\%}&0.63\%$\uparrow$\\
       
        \hline
        
    \end{tabular}
    \label{selection result tab}
\end{table}
\footnotetext[6]{Notice that some OOD accuracy are the same in the two methods since the same model is selected. This happens when the unseen domain is close to $\Eava$ so that the validation accuracy metric is close to ours.}

\paragraph{Result}
We summarize our experimental results in Table~\ref{selection result tab}. For each environment setup, we select the best model according to Algorithm~\ref{Model Selection} and validation accuracy. 
The results show that on all datasets, our selection criterion significantly outperforms the validation accuracy in average OOD accuracy. 
For a more detailed comparison, our method improves the OOD accuracy in most of the 12 setups.
Our experiments demonstrate that our criterion can help select models with higher OOD accuracy.

\subsection{Learnability of Real-World OOD Problem}\label{sec_assumption_example}

\begin{figure}[htp]
    \centering
    \includegraphics[width=0.95\textwidth]{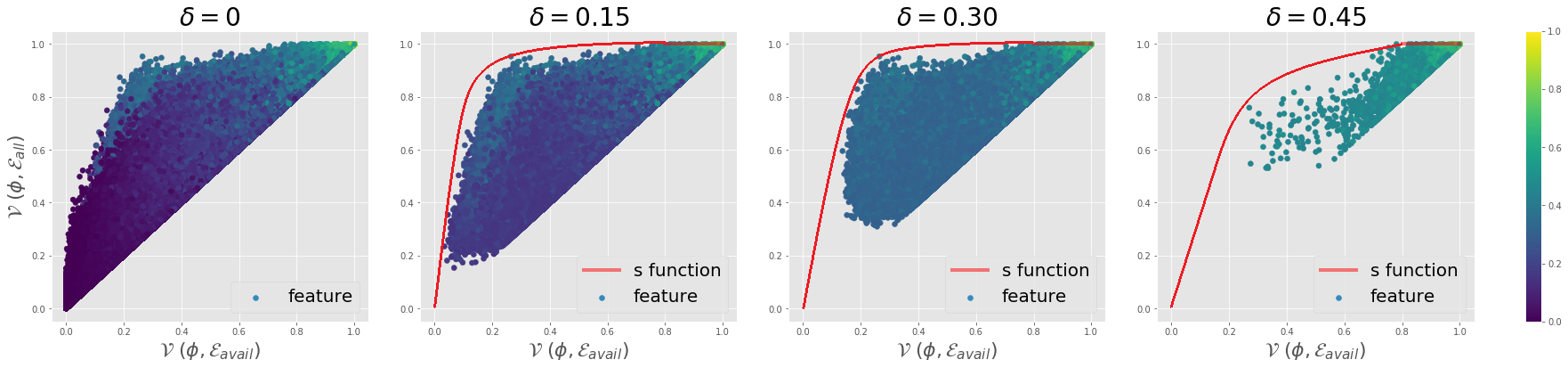}
    \caption{The expansion function of the OOD generalization problem on Office-Home. The x-axis stands for $\Vcal(\phi,\Eava)$ and the y-axis for $\Vcal(\phi,\Eall)$.
    There are approximately 2 million points in each image, with each point representing a feature, and its color represents its informativeness. The solid red line stands for the expansion function under the corresponding $\delta$. When $\delta$ increases, the expansion function decreases. When $\delta =0$, no expansion function can make it learnable.}
    \label{fig_OH_expan_func}
\end{figure}

One may wonder if the expansion function really exists and what it will look like for a real-world OOD generalization task.
In this section, we consider the OOD dataset Office-Home \cite{venkateswara2017deep}.
We explicitly plot \emph{millions of} features' $\Vcal_\rho(\phi,\Eava)$ and $\Vcal_\rho(\phi,\Eall)$ with Total Variation $\rho$ to see what the expansion function is like in this task.
We take the architecture as ResNet-50 \cite{He_2016_CVPR}, and we trained thousands of models with more than five algorithms, obtaining about 2 million features.
The results are in Figure~\ref{fig_OH_expan_func}.

\paragraph{Existence of $s(\cdot)$.}
When $\delta=0$, some non-informative features are nearly $0$-invariant across $\Eava$ but are varying across $\Eall$, so no expansion function can make this task learnable, i.e., this task is NOT $(s(\cdot),0)$ for any expansion function.
But as $\delta$ increases, only informative features are left, and now we can find appropriate $s(\cdot)$ to make it learnable.
We can clearly realize from the figure that $s(\cdot)$ do exist when $\delta\geq 0.15$.

\paragraph{Trade-off between $s(\cdot)$ and $\delta$.}
The second phenomenon is that the slope of $s(\cdot)$ decreases as $\delta$ increases, showing a trade-off between $s(\cdot)$ and $\delta$.
Although this trade-off comes naturally from the definition of learnability, it has a deep meaning.
As is shown in Section~\ref{section_bound}, $\text{err}(f)$ is bounded by $O(s(\varepsilon))$ where $\varepsilon$ is the variation of the model.
To make the bound tighter, a natural idea is to choose a flatter $s(\cdot)$.
However, a flatter $s(\cdot)$ corresponds to a larger $\delta$.
Typically, learning a model to meet this higher informativeness requirement is more difficult, and it is possible that the algorithm achieves this by capturing more domain-specific features, which will therefore increase the variation of the model, $\varepsilon$.
As a result, we are not sure whether $s(\varepsilon)$ will increase or decrease.
We believe this is also the essence of model selection: i.e., \emph{to trade-off between the variation and informativeness of a model}, which is done in Formula~\ref{formula_criterion}.

\section{More Related Works}\label{related_work}

Domain generalization \cite{blanchard2011generalizing,muandet2013domain}, or OOD generalization, has drawn much attention recently \cite{gulrajani2020search,koh2020wilds}.
The goal is to learn a model from several training domains and expect good performance on unseen test domains.  
\cite{wang2021generalizing,zhou2021domain} offer a comprehensive survey.
A popular solution is to extract domain-invariant feature representation.
\cite{peters2016causal} and \cite{rojas2018invariant} proved that when the model is linear, the invariance under training domains can help discover invariant features on test domains.
\cite{arjovsky2019invariant} introduces the invariant prediction into neural networks and proposes a practical objective function.
After that, a lot of works arise from the view of causal discovery, distributional robustness and conditional independence \cite{ahuja2020invariant,bai2020decaug,creager2021exchanging,chang2020invariant,jin2020domain,koyama2020out,krueger2020out,parascandolo2020learning,sagawa2019distributionally,xie2020risk}.
On the other hand, some works point out the weakness of existing methods from the theoretical and experimental perspectives \cite{ahuja2020empirical,gulrajani2020search,kamath2021does,nagarajan2020understanding,rosenfeld2020risks}.
%In this paper, we start from the invariant features over the available domains and analyze the predictive performance on unseen test domains.Our results do not depend on any causal mechanism of the ground truth model nor a hyper-distribution of the domains.Thus this work verifies the benefits of learning invariant features for domain generalization from a new perspective.

The OOD generalization requires restrictions on how the target domains may differ.
A straightforward approach is to define a set of test domains around the training domain using some distribution distance measure \cite{bagnell2005robust,blanchet2019data,esfahani2018data,hu2013kullback,sagawa2019distributionally,shafieezadeh2015distributionally,sinha2018certifying,xie2020risk}.
Another feasible route is the causal framework which is robust to the test distributions caused by interventions\cite{pearl2009causality,peters2017elements} on variables, e.g.,
\cite{arjovsky2019invariant,heinze2021conditional,magliacane2017domain,meinshausen2018causality,muller2020learning,pfister2019stabilizing,rojas2018invariant,scholkopf2012causal}.
The principle of these methods is that a causal model is invariant and can achieve the minimal worst-case risk \cite{aldrich1989autonomy,haavelmo1944probability,pearl2009causality,rojas2018invariant}.
Since the test distribution is unknown, additional assumptions are required for generalization analysis.  \cite{blanchard2011generalizing,deshmukh2019generalization,muandet2013domain} assume that the domains are generated from a hyper-distribution and measures the average risk estimation error bound.
\cite{albuquerque2019adversarial} derives a
risk bound for any linear combination of training domains. 
For more related results in domain adaptation, a closed field where the test domains can be seen but are unlabeled, please see \cite{ben2007analysis,ben2010theory,johansson2019support}. %Different from the existing methods, we introduce an expansion function to directly characterize the discrepancy between training distributions and test distributions and then derive a novel OOD generalization bound based on the expansion function.

\section{Conclusion}\label{sec_conclusion}
In this paper, we take the first step towards a rigorous theoretical framework of OOD generalization. We propose a mathematical formulation to characterize the learnability of OOD generalization problem. Based on our framework, we prove generalization bounds and give guarantees for OOD generalization error. Inspired by our bound, we design a model selection criterion to check the model's variation and validation accuracy simultaneously. Experiments show that our metric has a significant advantage over the traditional selection method.

\section{Appendix: Illustration of Model's Variation}\label{appsec1}

In this section, we illustrate why we need to define the variation of a model $f$ as
\benrr
\Vcal^{\text{sup}}(h, \Eava) &\triangleq &\sup_{\beta \in \mathcal S^{d-1}} \Vcal(\beta^\top h, \Eava),\\
\eenrr
where $(\beta^\top h)(x) = \beta^\top h(x)$ and $\Scal^{d-1} = \{\beta \in \mathbb R^d: \|\beta\|_2 = 1\}$ is the unit $(d-1)$-sphere.

\begin{figure}[htp]
    \centering
    \includegraphics[width=0.5\textwidth]{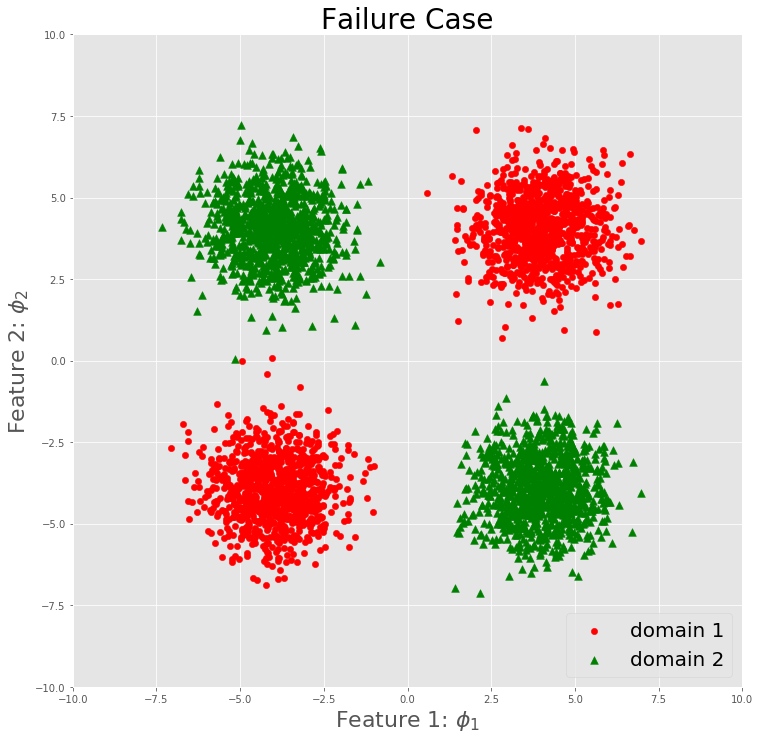}
    \caption{The Failure Case.}
    \label{FigA1}
\end{figure}

One seemingly plausible definition of the variation of a classifier $f$ can be the supremum over all $\Vcal(\phi_i,\Eava),i\in[d]$. However, there exist cases where $\Vcal(\phi_i,\Eall) = 0,\forall i\in[d]$ but the distribution of $h$ varies a lot in $\Eava$. 
We give a concrete failure case here. 

Consider a binary classification task with $\mathcal Y =\{-1,1\}$ and let $d=2$. Assume we learn a feature extractor $h = (\phi_1, \phi_2)^\top$ such that for a given label $y$, the distributions of $h$ under two domains are
% \benrr
% \text{Domain 1: } &h \sim 0.5 \Ncal\big(y(4,4)^\top, \mathbf I_2 \big) + 0.5 \Ncal\big(y(-4,-4)^\top, \mathbf I_2 \big),\\
% \text{Domain 2: }  &h \sim 0.5 \Ncal\big(y(4,-4)^\top, \mathbf I_2 \big) + 0.5 \Ncal\big(y(-4,4)^\top, \mathbf I_2 \big).
% \eenrr
\benrr
\text{Domain 1: } & y \sim \text{unif}\{+1,-1\}, & h|y \sim  \Ncal\big(y(4,4)^\top, \mathbf I_2 \big) \\
\text{Domain 2: }  & y \sim \text{unif}\{+1,-1\}, & h|y \sim  \Ncal\big(y(4,-4)^\top, \mathbf I_2 \big).
\eenrr

It is easy to see that the marginal distributions of both features alone are identical across the two domains. However, the distributions of $h$ are different (nearly separate at all). The empirical distributions of the two domains are present in Figure~\ref{FigA1}. This example shows that merely control the supremum of $\Vcal(\phi_i,\Eava),i\in[d]$ is not enough to control the Total Variation of two domains' density, and so it is not enough to upper bound the $\text{err}(f)$. To do so, we need a stronger quantity like $\Vcal^{\text{sup}}(h,\Eava).$

\section{Appendix: Proofs}\label{appsec2}

In this section, we provide complete proofs of our three bounds.

\subsection{Proof of Theorem~4.1}

{\bf Theorem} {\it
%Consider the generalization problem from $\Eava$ to $\Eall$ is $(s(\cdot),\delta)$-learnable under $\Phi$ with the total variation distance $\rho$\footnote{For two distribution $\mathbb P,\mathbb Q$ with probability density function $p,q$, $\rho(\mathbb P,\mathbb Q) = \frac12\int_x |p(x) - q(x)|\mathrm dx$}.
Let the loss function $\ell(\cdot,\cdot)$ be bounded by $[0,C].$ We denote 
\benrr
\Vcal^{\text{sup}}(h, \Eava) &\triangleq &\sup_{\beta \in \mathcal S^{d-1}} \Vcal(\beta^\top h, \Eava),\\
\Ical^{\text{inf}}(h, \Eava) &\triangleq &\inf_{\beta \in \mathcal S^{d-1}} \Ical(\beta^\top h, 
\Eava),
\eenrr
where $(\beta^\top h)(x) = \beta^\top h(x)$ is a feature and $\Scal^{d-1} = \{\beta \in \mathbb R^d: \|\beta\|_2 = 1\}$ is the unit $(d-1)$-sphere. 
Suppose we have learned a classifier $f(x) = g(h(x))$ such that for any $e \in \Eall$ and $y \in \Ycal$, $p_{h^e|Y^e} (h|y) \in L^2(\Rbb^d).$
%and $\hat p_{h^e|Y^e} (t|y) \in L^1(\Rbb^d)$
Denote the characteristic function of random variable $h^e|Y^e$ as 
\benrr
\hat p_{h^e|Y^e}(t|y) = \mathbb E [\exp\{i \langle t,h^e \rangle\}|Y^e=y].
\eenrr
Assume the hypothetical space $\mathcal F$ satisfies the following regularity conditions that $\exists \alpha,M_1,M_2 >0, \forall f \in \Fcal, \forall e \in \Eall, y\in\Ycal$,
\benr\label{eqA1}
\int_{h\in\R^d} p_{h^e|Y^e}(h|y) |h|^\alpha \mathrm d h\leq M_1 \quad \text{and} \quad \int_{t\in\R^d} \hat p_{h^e|Y^e}(t|y) |t|^\alpha \mathrm dt \leq M_2.
\eenr
If $(\Eava,\Eall)$ is $\big(s(\cdot),\Ical^{\text{inf}}(h,\Eava)\big)$-learnable under $\Phi$ with Total Variation $\rho$\footnote{For two distribution $\mathbb P,\mathbb Q$ with probability density function $p,q$, $\rho(\mathbb P,\mathbb Q) = \frac12\int_x |p(x) - q(x)|\mathrm dx$.}, then we have
\benrr
\err(f) \leq O\Big(s\big(\Vcal^{\text{sup}}_\rho (h,\Eava)\big)^{\frac{\alpha^2}{(\alpha+d)^2}}\Big).
\eenrr
Here $O(\cdot)$ depends on $d,C,\alpha,M_1,M_2$.
}

\begin{proof}
For any $e \in \Eava$ and $e' \in \Eall$,
\benrr
\Pbb_{Y}(y) = \Pbb_{Y^e}(y) =\Pbb_{Y^{e'}}(y).
\eenrr
We can decompose the loss gap between $e$ and $e'$ as
\benrr
&& \Ebb\big[ \ell(f(X^{e'}), Y^{e'}) \big] - \Ebb\big[ \ell(f(X^{e}), Y^{e}) \big] \\
&=& \Ebb\big[ \ell(g(h(X^{e'})), Y^{e'}) \big] - \Ebb\big[ \ell(g(h(X^{e})), Y^{e}) \big] \\
&=& \sum_{y=1}^K \Pbb_Y(y)\Big(\Ebb\big[ \ell(g(h(X^{e'})), Y^{e'}) \big| Y^{e'} = y \big] - \Ebb\big[ \ell(g(h(X^{e})), Y^{e}) \big| Y^{e} = y \big] \Big).
\eenrr
Therefore, to bound $\text{err}(f)$, it is sufficient to bound
$$
\Big|\Ebb\big[ \ell(f(X^{e'}), Y^{e'}) \big| Y^{e'} = y \big] - \Ebb\big[ \ell(f(X^{e}), Y^{e}) \big| Y^{e} = y \big] \Big|
$$
for any $y\in \Ycal, (e,e') \in (\Eava,\Eall).$
Given $y,e,e'$, we have
\benrr
&& \Big|\Ebb\big[ \ell(f(X^{e'}), Y^{e'}) \big| Y^{e'} = y \big] - \Ebb\big[ \ell(f(X^{e}), Y^{e}) \big| Y^{e} = y \big] \Big| \\
&\leq& C \int_{\R^d} \big| p_{h^{e'}|Y^{e'}}(h|y) - p_{h^{e}|Y^{e}}(h|y) \big| \mathrm d h = C * I
\eenrr
where $h^e$ represents the $d$-dimensional random vector $h(X^e)$ and 
\benrr
I = \int_{\R^d} \big| p_{h^{e'}|Y^{e'}}(h|y) - p_{h^{e}|Y^{e}}(h|y) \big| \mathrm d h.
\eenrr
In the following, we shall show that the term $I$ is upper bounded by $O\Big(s\big(\Vcal^{\text{sup}}(h, \Eava)\big)\Big).$

First, we decomposed the term $I$ into $I_1 + I_2$ where
\benrr
I_1 &=& \int_{|h|\leq r_1} \big| p_{h^{e'}|Y^{e'}}(h|y) - p_{h^{e}|Y^{e}}(h|y) \big| \mathrm d h \\
I_2 &=& \int_{|h| > r_1} \big| p_{h^{e'}|Y^{e'}}(h|y) - p_{h^{e}|Y^{e}}(h|y) \big| \mathrm d h.
\eenrr
Here $r_1$ is a scalar to be decided, and $|h|$ is the Euclidean norm of $h.$
% According to (\ref{concentration_assumption}), the density function satisfies that, for any $y$,
% \benrr
% \int p_{h^e|Y^e}(h|y) |h|_2^\alpha \mathrm d h \leq M_1, \quad \text{and} \quad \int \hat p_{h^e|Y^e}(t|y) |t|^\alpha dt \leq M_2,
% \eenrr
% where $|x|$ is the Euclidean norm of $x$, $\Cbf^0(\R^d)$ stands for all continuous functions on $\R^d$ and $\hat p$ is the Fourier transform of the density function $p_{h^e|Y^e}(h|y)$ with respect to $h$:
% \benrr
% \hat p_{h^e|Y^e}(h|y) = \int_{h \in \mathbb R^d} p_{h^e|Y^e}(h|y) e^{-i \langle t,h\rangle} \mathrm d h.
% \eenrr
According to (\ref{eqA1}), the term $I_2$ is bounded above:
\benrr
I_2 &\leq& \int_{|h| > r_1} \big| p_{h^{e'}|Y^{e'}}(h|y) - p_{h^{e}|Y^{e}}(h|y) \big| |h|^\alpha r_1^{-\alpha} \mathrm d h \\
&\leq& r_1^{-\alpha} \Big( \int_{h\in\R^d} \big| p_{h^{e'}|Y^{e'}}(h|y) \big| |h|^\alpha \mathrm d h +  \int_{h\in\R^d} \big| p_{h^{e}|Y^{e}}(h|y) \big| |h|^\alpha \mathrm d h \Big) \\
&\leq& 2 M_1 r_1^{-\alpha}.
\eenrr
Next we deal with $I_1.$ Since $p_{h^{e'}|Y^{e'}}\in L^1(\R^d)$ and $\hat p_{h^{e}|Y^{e}} \in L^1(\R^d)$,
\benrr
p_{h^{e}|Y^{e}}(h|y) = \int_{t \in \R^d} e^{-i\langle t, h\rangle} \hat p_{h^{e}|Y^{e}}(t|y) \mathrm d t.
\eenrr
Then we have

\benrr
&& \big|p_{h^{e'}|Y^{e'}}(h|y) - p_{h^{e}|Y^{e}}(h|y)\big| \\
&\leq& \Big| \int_{t \in \R^d} \exp(-i\langle t, h\rangle) \big( \hat p_{h^{e'}|Y^{e'}}(t|y) - \hat p_{h^{e}|Y^{e}}(t|y) \big) \mathrm d t \Big| \\
&\leq&   \int_{t \in \R^d} \big| \hat p_{h^{e'}|Y^{e'}}(t|y) - \hat p_{h^{e}|Y^{e}}(t|y) \big| \mathrm d t \\
&\leq& \int_{|t| \leq r_2} \big| \hat p_{h^{e'}|Y^{e'}}(t|y) - \hat p_{h^{e}|Y^{e}}(t|y) \big| \mathrm d t \\
&& + r_2^{-\alpha}  \int_{|t|>r_2} \big| \hat p_{h^{e'}|Y^{e'}}(t|y) - \hat p_{h^{e}|Y^{e}}(t|y) \big| |t|^{\alpha} \mathrm d t \\
&\leq& \int_{|t| \leq r_2} \big| \hat p_{h^{e'}|Y^{e'}}(t|y) - \hat p_{h^{e}|Y^{e}}(t|y) \big| \mathrm d t + 2 M_2 r_2^{-\alpha}.
\eenrr
Plugging the above upper bound into $I_1$,
\benrr
I_1 &\leq& \int_{|h|\leq r_1} 
\int_{|t| \leq r_2} \big| \hat p_{h^{e'}|Y^{e'}}(t|y) - \hat p_{h^{e}|Y^{e}}(t|y) \big| \mathrm d t
\mathrm d h + \int_{|h|\leq r_1} 2 M_2 r_2^{-\alpha} \mathrm d h \\
&\leq& \frac{\pi^{d/2}}{\Gamma(d/2+1)} r_1^d \times I_3 
+ \frac{2 M_2 \pi^{d/2}}{\Gamma(d/2+1)} r_1^d r_2^{-\alpha}
\eenrr
where 
\benrr
I_3 = \int_{|t| \leq r_2} \big| \hat p_{h^{e'}|Y^{e'}}(t|y) - \hat p_{h^{e}|Y^{e}}(t|y) \big| \mathrm d t.
\eenrr
Note that $p_{h^{e}|Y^{e}}(t|y) \in L^1(\R^d) \cap L^2(\R^d).$ By the Projection theorem \cite{natterer2001mathematics,korostelev2012minimax},
\benr\label{RandonTrans}
\widehat{\Rcal_e}(\beta, u) = \hat p_{h^{e}|Y^{e}}(u \beta|y), \quad u \in \R, \,\, \beta \in S^{d-1}, 
\eenr
where $\Rcal_e(\beta, u)$ is the Radon transform of $p_{h^{e}|Y^{e}}(t|y)$:
\benrr
\Rcal_e(\beta, u) = \int_{h: \langle h, \beta \rangle = u} p_{h^{e}|Y^{e}}(h|y) \mathrm d v
\eenrr
and $\widehat{\Rcal_e}(\beta, w)$ is the Fourier transform of $\Rcal_e(\beta, u)$ with respect to $u$:
\benrr
\widehat{\Rcal_e}(\beta, w) = \int_{u\in\R} \exp(iuw) \Rcal_e(\beta, u) du.
\eenrr
% The equation (\ref{RandonTrans}) is an essential tool in tomography, can be found
% in \cite{natterer2001mathematics} and \cite{korostelev2012minimax}.
Thus we can rewrite the term $I_3$ as
\benrr
I_3 &=& \int_{\beta \in S^{d-1}} \int_{|w|\in [0, r_2]} |w|^{d-1}  \big| \widehat{\Rcal_{e'}}(\beta, w) - \widehat{\Rcal_e}(\beta, w)\big| \mathrm d w \mathrm d \beta \\
&\leq& r_2^{d-1} \int_{\beta \in S^{d-1}} \int_{|w|\in [0, r_2]} \big| \widehat{\Rcal_{e'}}(\beta, w) - \widehat{\Rcal_e}(\beta, w)\big| \mathrm d w \mathrm d s\beta\\
&\leq& r_2^{d-1} \int_{\beta \in S^{d-1}} \int_{|w|\in [0, r_2]} \int_{u \in \R} \big|\Rcal_{e'}(\beta, u) - \Rcal_e(\beta, u) \big| \mathrm d u \mathrm d w \mathrm d \beta.
\eenrr

Since the problem is $(s(\cdot),\Ical^{\text{inf}}(h,\Eava))$-learnable, and $\forall \beta \in \mathcal S^{d-1}$, the informativeness of feature $\beta^\top h$ is lower bounded by 
$$
\Ical(\beta^\top h,\Eava)\geq \Ical^{\text{inf}}(h,\Eava),
$$
we know that for any $\beta \in\mathcal S^{d-1}$,
$$
\Vcal(\beta^\top h, \Eall) \leq s\big(\Vcal(\beta^\top h, \Eava)\big).
$$
Therefore, we have 
\benrr
\Vcal^{\text{sup}}(h,\Eall) &=& \sup_{\beta \in \mathcal S_{d-1}} \Vcal(\beta^\top h, \Eall)  \\
&\leq&  \sup_{\beta \in \mathcal S_{d-1}} s\big(\Vcal(\beta^\top h, \Eava)\big) = s\big(\Vcal^{\text{sup}}(h,\Eava)\big).
\eenrr
Note that, for any given $\beta$, $\Rcal_e(\beta, u)$ is the probability density of the projected feature $\beta^\top h.$ 
So for any $e', e \in \Eall$,
\benrr
\int_{u \in \R} \big|\Rcal_{e'}(\beta, u) - \Rcal_e(\beta, u) \big| \mathrm d u \leq s(\Vcal^{\text{sup}}(h, \Eava)).
\eenrr
Therefore,
%Since $\int_{u \in \R} \big|\Rcal_{e'}(\beta, u) - \Rcal_e(\beta, u) \big| \mathrm d u \leq s(\Vcal^{\text{sup}}(h, \Eava))$, 
\benrr
I_3 \leq 2 r_2^d \times \frac{\pi^{d/2}}{\Gamma(d/2+1)} \times s(\varepsilon).
\eenrr
Combining the result of $I_1$, $I_2$ and $I_3$, we have
\benrr
I \leq \frac{2 \pi^d}{\Gamma^2(d/2+1)} r_1^d r_2^d s(\Vcal^{\text{sup}}(h, \Eava)) + \frac{2 M_2 \pi^{d/2}}{\Gamma(d/2+1)} r_1^d  r_2^{-\alpha} + 2M_1 r_1^{-\alpha}.
\eenrr
We take
\benrr
r_1 = M_1^{\frac{1}{\alpha+d}} M_2^{-\frac{d}{(\alpha+d)^2}} s(\Vcal^{\text{sup}}(h, \Eava))^{-\frac{\alpha}{(\alpha+d)^2}},
\eenrr
and
%\quad \text{and} \quad 
\benrr
r_2 = M_2^{\frac{1}{\alpha+d}} s(\Vcal^{\text{sup}}(h, \Eava))^{-\frac{1}{\alpha+d}}.
\eenrr
Hence
\benrr
I \leq \Big( \frac{2\pi^d}{\Gamma^2(d/2+1)} + \frac{2 \pi^{d/2}}{\Gamma(d/2+1)} +2 \Big) M_1^{\frac{d}{\alpha+d}} M_2^{\frac{\alpha d}{(\alpha+d)^2}} s(\Vcal^{\text{sup}}(h, \Eava))^{\frac{\alpha^2}{(\alpha+d)^2}}.
\eenrr
The proof is finished.
\end{proof}

\subsection{Proof of Theorem~4.2}\label{appsec3}

{\bf Theorem} {\it Consider any loss satisfying $\ell(\hat y,y) = \sum_{k=1}^K \ell_0(\hat y_k , y_k).$
Let the loss function $\ell_0(\cdot,\cdot)$ be bounded by $[0,C].$

For any classifier with linear top model $g$, i.e., 
\benrr
f(x) = A h(x) + b \quad \text{with} \quad  A\in \Rbb^{K\times d},\,\, b \in \Rbb ^{K},
\eenrr
if $(\Eava,\Eall)$ is $\big(s(\cdot),\Ical^{\text{inf}}(h,\Eava)\big)$-learnable under $\Phi$ with Total Variation $\rho$, then we have
\benr
\err(f) \leq  O\Big(s\big(\Vcal^{\text{sup}}(h,\Eava)\big)\Big).
\eenr
Here $O(\cdot)$ depends only on $d$ and $C$.
}

\begin{proof}

For any $e \in \Eava$ and $e' \in \Eall$, we know that
$
\Pbb_{Y}(y) = \Pbb_{Y^e}(y) =\Pbb_{Y^{e'}}(y).
$
Furthermore the generalization gap between $e$ and $e'$ is
\benrr
&& \Ebb\big[ \ell(f(X^{e'}), Y^{e'}) \big] - \Ebb\big[ \ell(f(X^{e}), Y^{e}) \big] \\
&=& \sum_{y=1}^K \Pbb_Y(y)\Big(\Ebb\big[ \ell(f(X^{e'}), Y^{e'}) \big| Y^{e'} = y \big] - \Ebb\big[ \ell(f(X^{e}), Y^{e}) \big| Y^{e} = y \big] \Big) \\
&=& \sum_{y=1}^K \Pbb_Y(y)\Big(\Ebb\big[ \sum_{j=1}^K \ell_0( f(X^{e'})_j, y_j ) \big| Y^{e'} = y \big] - \Ebb\big[ \sum_{j=1}^K \ell_0( f(X^{e})_j, y_j )  \big| Y^{e} = y \big] \Big) \\
&=& \sum_{y=1}^K \sum_{j=1}^K \Pbb_Y(y)\Big(\Ebb\big[ \ell_0( f(X^{e'})_j, y_j ) \big| Y^{e'} = y \big] - \Ebb\big[ \ell_0( f(X^{e})_j, y_j )  \big| Y^{e} = y \big] \Big),
\eenrr
where $f(X^{e'})_j = A_j h(x) + b_j.$ Here $A_j$ is the $j$-th row of the matrix $A$ and $b_j$ stands for the $j$-th element of the vector $b.$
Then it suffices to uniformly bound
\benrr
&& \Big| \Ebb\big[ \ell_0( f(X^{e'})_j, y_j ) \big| Y^{e'} = y \big] - \Ebb\big[ \ell_0( f(X^{e})_j, y_j )  \big| Y^{e} = y \big] \Big| \\
&=&  \Big| \int_{\R^d} \ell_0 (  A_j h + b_j, y )  \big( p_{h^{e'}|Y^{e'}}(h|y) - p_{h^{e}|Y^{e}}(h|y) \big) \mathrm d h \Big|, %\\
% &\leq& C \times \int_{\R^d}  \big| p_{h^{e'}|Y^{e'}}(h|y) - p_{h^{e}|Y^{e}}(h|y) \big| \mathrm d h \\
% &\leq& C \times s(\Vcal^{\text{sup}}(h, \Eava)).
\eenrr
where $h^e$ is the $d$-dimensional random vector $h(X^e).$ 
Let $t = A_j h + b_j.$ Then,
\benrr
&&  \Big| \int_{\R^d} \ell_0 (  A_j h + b_j, y )  \big( p_{h^{e'}|Y^{e'}}(h|y) - p_{h^{e}|Y^{e}}(h|y) \big) \mathrm d h \Big| \\
&=& \Big| \int_{t \in \R} \int_{\frac{A_j}{\|A_j\|_2} h + \frac{b_j}{\|A_j\|_2} = \frac{t}{\|A_j\|_2}} \ell_0 ( t, y )  \big( p_{h^{e'}|Y^{e'}}(h|y) - p_{h^{e}|Y^{e}}(h|y) \big) \mathrm d h \mathrm d t \Big| \\
&\leq& C \times \Big| \int_{t \in \R} \Rcal_{e'}(\frac{A_j}{\|A_j\|_2}, \frac{t-b_j}{\|A_j\|_2}) - \Rcal_{e}(\frac{A_j}{\|A_j\|_2}, \frac{t-b_j}{\|A_j\|_2}) \mathrm d t \Big| \\
&\leq& O\big(s(\Vcal^{\text{sup}}(h, \Eava)) \big).
\eenrr
Hence 
\benrr
\Ebb\big[ \ell(f(X^{e'}), Y^{e'}) \big] - \Ebb\big[ \ell(f(X^{e}), Y^{e}) \big] \leq O\big(s(\Vcal^{\text{sup}}(h, \Eava)) \big).
\eenrr

\end{proof}

\subsection{Proof of Theorem~4.3}\label{appsec4}

{\bf Theorem} {\it Consider $0$-$1$ loss: $\ell(\hat y, y) = \mathbb I(\hat y \neq y)$. 
For any $\delta > 0$ and any expansion function satisfying 1) $s'_+(0) \triangleq \lim_{x\to0^+ } \frac {s(x) - s(0)}{x} \in (1,+\infty)$; 2) exists $k>1, t >0$, s.t. $kx \leq s(x) < +\infty, x \in [0,t]$,
%For any linear expansion function $s(x) = kx,k\in(1,+\infty)$ 
there exists a constant $C_0$ and an OOD generalization problem $(\Eava,\Eall)$ that is $(s(\cdot),\delta)$-learnable under linear feature space $\Phi$ w.r.t symmetric KL-divergence $\rho$, s.t. $\forall \varepsilon\in [0,\frac t 2]$, the optimal classifier $f$ satisfying $\Vcal^{\text{sup}}(h,\Eava) = \varepsilon$ will have the OOD generalization error lower bounded by 
\benr
\err(f)\geq C_0\cdot s(\Vcal^{\text{sup}}(h,\Eava))
\eenr
}

\begin{proof}
The expansion function $s(x)$ satisfies $kx \leq s(x) < +\infty$, $x \in [0,t].$
Construct an another function as:
\benrr
\tilde s(x) = \begin{cases} k x & x\leq t  \\ s(x) & x > t \end{cases}.
\eenrr
Clearly, $\tilde s(\cdot)$ is also an expansion function.
According to Lemma~\ref{lemma_lower_linear}, for $(\tilde s(x),\delta)$, there exists a constant $C_1 > 0$ and $(\Eava,\Eall)$ that is $(\tilde s(\cdot),\delta)$, s.t. for any $\Vcal^{\text{sup}}(h,\Eava) \leq \frac t 2$, the optimal classifier $f$ satisfies
$$
\text{err}(f) \geq C_1 \tilde s(\Vcal^{\text{sup}}(h,\Eava)) = C_1 k_1 \Vcal^{\text{sup}}(h,\Eava).
$$
Then it suffices to find a constant $C'_0$ such that
\benrr
\Vcal^{\text{sup}}(h,\Eava) \geq C'_0 s(\Vcal^{\text{sup}}(h,\Eava)).
\eenrr
Notice that $s'_{+}(0) = M' \in (1, +\infty).$ Thus there exists $\delta$ such that $\forall x \in [0, \delta]$, $\frac{s(x)}{x} \leq 2M'.$ 
In addition, $s(x) \leq M$, $x\in [0, t/2].$ 
Then, for any $x \geq \delta$, $\frac{x}{s(x)} \geq \frac{\delta}{M}.$ Let $C'_0 = \max\{ \frac{\delta}{M}, \frac{1}{2M'}\}.$ So, for any $x \in [0, t/2]$, $x \geq C'_0 s(x).$ The proof is finished.
\end{proof}

\begin{lemma}[lower bound for linear expansion function]\label{lemma_lower_linear}
Consider $0$-$1$ loss $\ell(\hat y,y) = \mathbb I(\hat y \neq y)$. 
For any linear expansion function $s(x) = kx$, $x \in [0,t]$, $k \in (1, +\infty)$
and any $\delta>0$,
%that is linear in $[0, t]$, i.e., $s(x) = k x$, $k \in (1,+\infty)$ and any $\delta>0$, 
there exists a constant $C_1$ and an OOD generalization problem $(\Eava,\Eall)$ that is $(s(\cdot),\delta)$-learnable under linear feature space $\Phi$ with symmetric KL-divergence $\rho$, s.t. $\forall \varepsilon\leq \frac t 2$, the optimal classifier $f$ satisfying $\Vcal^{\text{sup}}(h,\Eava) = \varepsilon$ have $\text{err}(f)$ bounded by 
\benr\label{linear_inequality}
\text{err}(f) \geq C_1 \cdot  s(\Vcal^{\text{sup}}(h,\Eava)).
\eenr
\end{lemma}

\begin{proof}
We construct $(\Eava,\Eall)$ as a binary classification task, where there are two domains in $\Eava$, denoted as $\{1,2\}$, and other two domains in $\Eall \setminus \Eava$, denoted as $\{3,4\}$. 
%we further denote $a_1 = 0,a_2 = 1,a_3 = \sqrt{k}$.
The dataset $(x,y)$ for domain $e \in \Eall$ is constructed as
$$
y \sim \text{unif}\{-1,1\}, \,\, z\sim \mathcal N(r y,1), \,\, \eta^e \sim \mathcal N(a_e y,1), \,\, x^e = \begin{pmatrix} z \\ \eta^e \end{pmatrix}.
$$
Here we set
\benrr
a_1 = -\sqrt{\frac{t}{2}},\,\, a_2 =  \sqrt{\frac{t}{2}},\,\, a_3 =- \sqrt{ \frac{kt}{2} }, \,\, a_4 = \sqrt{ \frac{kt}{2} }, \,\, r = \sqrt{t }.
\eenrr
For any $\mathbf w = (w_1,w_2)^\top$, the distribution of $\phi^e = \mathbf w^\top x^e$ given $y$ is
$$
\phi^e|y \sim \mathcal N \big(y(w_1 r + w_2 a_e),\|\mathbf w \|^2 \big).
$$
Now we calculate the variation of the feature.
Notice that the symmetric KL divergence $\rho$ of two Gaussian distributions $\Pbb_1 \sim N(\mu_1, \sigma^2)$ and $\Pbb_2 \sim N(\mu_2, \sigma^2)$ is
\benrr
\rho(\Pbb_1, \Pbb_2) &=& \frac{1}{2} D_{KL}(\Pbb_1\|\Pbb_2) + \frac{1}{2} D_{KL}(\Pbb_2\|\Pbb_1)\\
&=& \frac{1}{2} \frac{1}{\sigma^2} (\mu_1-\mu_2)^2 .
\eenrr
Therefore, we have
\benrr
\Vcal(\phi^e, \Eava) = \sup_{y \in \{-1,1\}} \frac{w_2^2|a_1-a_2|^2}{2\|\wbf\|_2^2} = \frac{t w_2^2}{\|\wbf\|_2^2} \leq t,
\eenrr
and
\benrr
\Vcal(\phi^e, \Eall) = \sup_{y \in \{-1,1\}} \sup_{e, e'} \frac{w_2^2|a_e-a_{e'}|^2}{2\|\wbf\|_2^2} = k \frac{t w_2^2}{\|\wbf\|_2^2}.
\eenrr
Thus, for any $\phi \in \Phi$,
\benrr
s(\Vcal(\phi, \Eava)) = k \frac{t w_2^2}{\|\wbf\|_2^2} \geq \Vcal(\phi, \Eall).
\eenrr
Therefore the OOD generalization problem $(\Eava,\Eall)$ that is $(s(\cdot),\delta)$-learnable under linear feature space $\Phi$ with symmetric KL-divergence $\rho.$

\paragraph{Optimal Classifier} Now we consider $h(x) = (\phi_1(x), \ldots ,\phi_d(x))^\top$ such that $\Vcal^{\text{sup}}(h,\Eava) = \varepsilon \leq \frac t 2$, and see what the optimal classifier is like.
Let $\wbf_i=(w_{i1}, w_{i2})^\top$ be the coefficients of $\phi_i.$ 

If for any $i \in [d]$, $w_{i2}=0.$ Then $s(\Vcal^{\text{sup}}(\phi, \Eava)) = 0$ and for any $f$, $\err(f)=0.$ So inequality~\ref{linear_inequality} holds.

Now suppose there exists $i_0 \in [d]$ such that $w_{i_02} \neq 0.$ Without loss of generality, we assume $i_0=1$ and $\|\wbf_i\|\neq 0$ for any $i \in [d].$ 
We then claim that $\forall i\in [d], \exists c_i \in \Rbb , \wbf_i = c_i \wbf_1.$
Otherwise, there exists a normalized vector $\beta\in\R^d$ such that $\beta^\top h(x) = c' (0,1)^\top x$, and we have $\Vcal(\beta^\top h,\Eava) = t $, which is contradictory to $\Vcal^{\text{sup}}(h,\Eava) \leq \frac t 2.$

Since $\phi_i = c_i \phi_1$, it is obvious that under any loss function, the loss of optimal classifier on $h$ is the same as the optimal classifier on $\phi_1.$ In the following, we shall focus on the optimal loss on $\phi_1.$

Without loss of generality, we further denote $\phi_1$ as $\phi(x) = (w_1, w_2) x$ and $w_1>0.$ 
Since $\Vcal^{\text{sup}}(h,\Eava) = \varepsilon \leq t/2$, $|w_2| \leq w_1$. In addition, we have $r > |a_e|, e\in \{1,2\}$. Therefore, $w_1 r + w_2 a_e > 0, \text{sign}(y(w_1r+w_2a_e)) = \text{sign}(y)$, and we can easily realize that for any $e \in \Eava$, the optimal classifier is $f(x) = \text{sign}( \phi(x) )$.

The loss of $f$ in $\Eava$ is
\benrr
\mathcal L(\Eava,f) &=& \max{e\in\{1,2\}}  \frac 12 \big[\Pbb[f(x^e)<0 | Y=1] + \Pbb[f(x)>0 | Y=-1] \big]\\
&=& \max{e\in\{1,2\}} \Pbb[f(x^e)<0 | Y=1] \\ 
&=& \max_{e\in\{1,2\}} \int_{-\infty}^0 \frac{1}{\sqrt{2\pi}\|\wbf\|} \exp\Big(-\frac{1}{2} \frac{\big(\phi - (w_1 r + w_2 a_e)\big)}{\|\wbf\|^2} \Big) d \phi \\
%&=& \int_{-\infty}^{-w_1 r + w_2 \sqrt{\frac{t}{2}}} \frac{1}{\sqrt{2\pi}\|\wbf\|} \exp\Big(-\frac{1}{2} \frac{\phi^2}{\|\wbf\|^2} \Big) d \phi \\
&=&  \max_{e\in\{1,2\}} \int^{+\infty}_{w_1 r + w_2 a_e} \frac{1}{\sqrt{2\pi}\|\wbf\|} \exp\Big(-\frac{1}{2} \frac{\phi^2}{\|\wbf\|^2} \Big) d \phi \\
&=&  \int^{+\infty}_{w_1 r - |w_2| \sqrt{\frac{t}{2}}} \frac{1}{\sqrt{2\pi}\|\wbf\|} \exp\Big(-\frac{1}{2} \frac{\phi^2}{\|\wbf\|^2} \Big) d \phi \\
&=&  \int^{+\infty}_{\hat w_1 r - |\hat w_2| \sqrt{\frac{t}{2}}} \frac{1}{\sqrt{2\pi}} \exp\Big(-\frac{1}{2} \phi^2 \Big) d \phi,
\eenrr
where $\hat w_1 = w_1/\|\wbf\|$ and $\hat w_2 = w_2/\|\wbf\|.$ 
Similarly, we have
\benrr
\mathcal L(\Eall,f) = \int^{+\infty}_{\hat w_1 r - |\hat w_2| \sqrt{\frac{kt}{2}}} \frac{1}{\sqrt{2\pi}} \exp\Big(-\frac{1}{2} \phi^2 \Big) d \phi.
\eenrr
Combined together, the OOD generalization error of the optimal classifier with $\Vcal^{\text{sup}}(h,\Eava)  = \varepsilon$ is
\benrr
\err(f) &=& \int^{+\infty}_{\hat w_1 r - |\hat w_2| \sqrt{\frac{kt}{2}}} \frac{1}{\sqrt{2\pi}} \exp\Big(-\frac{1}{2} \phi^2 \Big) d \phi - \int^{+\infty}_{\hat w_1 r - |\hat w_2| \sqrt{\frac{t}{2}}} \frac{1}{\sqrt{2\pi}} \exp\Big(-\frac{1}{2} \phi^2 \Big) d \phi \\
&=& \int^{\hat w_1 r - |\hat w_2| \sqrt{\frac{t}{2}}}_{\hat w_1 r - |\hat w_2| \sqrt{\frac{kt}{2}}} \frac{1}{\sqrt{2\pi}} \exp\Big(-\frac{1}{2} \phi^2 \Big) d \phi\\
&\geq& C(\sqrt{k}-1) \sqrt{\frac t 2}|\hat w_2| \\
&\geq& C(\sqrt{k}-1) \sqrt{\frac t 2}|\hat w_2|^2 \\
&=& \frac {C(\sqrt{k}-1) \sqrt{\frac t 2}}{kt} s(\Vcal^{\text{sup}}(h,\Eava)).
\eenrr
We finish our proof by choosing $C_1 =\frac {C(\sqrt{k}-1) \sqrt{\frac t 2}}{kt} .$

\end{proof}

\section{Appendix: Experiment on Colored MNIST}
\label{appsec5}

In this section, we conduct experiment on ColoredMNIST, a hand designed OOD dataset, to illustrate why validation accuracy fail to select a good model in OOD dataset.

\subsection{Colored MNIST}\label{appsec51}

The Colored MNIST \cite{arjovsky2019invariant} is a common-used synthetic dataset in OOD generalization problem. In the dataset, picture is labeled with $0$ or $1$, and it contains two color channels, one of which being $28\times28$ pixels gray scale image from MNIST \cite{lecun1998gradient} while the other being a zero matrix. 
Let the grayscale image and the colored image be $X$ and $\tilde X$ respectively, i.e., $\tilde X = [X, 0]^\top$ and $\tilde X = [0,X]^\top$ correspond to red and green image. Given a domain $e \in [0,1]$, for an original image $X$ with the label $\hat Y=\Ibb\{\text{digit}<=4\}$, the data point in Colored MNIST is constructed with
\benr
Y = \begin{cases}\hat Y & \text{w.p. }0.75 \\ 1 -  \hat Y & \text{w.p. }0.25 \end{cases}, \,\, \tilde X^e = \begin{cases}[X,0]^\top & \text{w.p. } e+(1-2e)Y\\ [0,X]^\top & \text{w.p. } e+(1-2e)(1-Y)\end{cases}\label{CMNIST_formula}
\eenr
According to (\ref{CMNIST_formula}), the digit shape is invariant over domains, and the color is varying but might be more informative than the digit shape in some domains. The difficulty of OOD generalization is that we need to avoid learning color, since in $e\in\Eall$ the relationship between $e$ and $y$ might be entirely reversed. 

\subsection{Learnability of Colored MNIST}
As a warm-up, We first prove that for any $\delta$, Colored MNIST is a $(s(\cdot),\delta)$-learnable OOD problem under any feature space $\Phi$ with the total variation distance $\rho$, where
\benr
s(\varepsilon) = \frac{\max_{e, e' \in \Eall} |e-e'| }{\max_{e, e' \in \Eava} |e-e'|} \varepsilon.
\eenr
Here we assume the original dataset MNIST is generated from a distribution.

\begin{proof}
Denote $\phi(\tilde X^e)$ as $\phi^e $ and the density of $\phi(\tilde X^e) |Y^e = y$ as $p_{e,y}(x)$. In addition, denote the density of $\phi([X,0]^\top )|Y^e=y$ as $p^1_{e,y}(x)$ and the density of $\phi([0,X]^\top )|Y^e=y$ as $p^2_{e,y}(x).$
Therefore, we have $\forall e,y$,
\benrr
p_{e,y}(x) = [e+(1-2e)y]p^1_{e,y}(x) + [e+(1-2e)(1-y)]p^2_{e,y}(x)
\eenrr
Since the distance $\rho(\cdot,\cdot)$ is total variation, we know that for any two domains $e, e'$, 
\benrr
&& \rho \big(\Pbb(\phi^e|Y^e=y), \Pbb(\phi^{e'}|Y^{e'}=y) \big) \\
&=& \frac 1 2 \int \big|p_{e,y}(x) - p_{e',y}(x)\big| \mathrm dx\\
&=& \frac 1 2  \int \Big|[e+(1-2e)y]p^1_{e,y}(x) + [e+(1-2e)(1-y)]p^2_{e,y}(x) \\
&& \,\,\,\, -[e'+(1-2e')y]p^1_{e',y}(x) - [e'+(1-2e')(1-y)]p^2_{e',y}(x)\Big| \mathrm d x
\eenrr
Notice that $X$ is invariant across domains. Thus for all $x,y$,
\benrr
p^1_{e,y}(x) = p^1_{e',y}(x) ,
p^2_{e,y}(x) = p^2_{e',y}(x).
\eenrr
We can omit the subscript $e$ and 
\benrr
&& \rho \big(\Pbb(\phi^e|Y^e=y), \Pbb(\phi^{e'}|Y^{e'}=y) \big) \\
&=& \frac 1 2 \int\Big| (e-e')(1-2y) p^1_{y}(x) - (e-e')(2y-1)p^2_{y}(x) \Big| \mathrm dx \\
&=& |e-e'| \int \Big|p^1_{y}(x) - p^2_{y}(x)\Big|\mathrm dx \\
&=& C |e-e'|,
\eenrr
where $C$ is a constant independent to $e,\delta.$ By choosing $e,e'$ separately in $\Eava$ and $\Eall$, we can derive the expansion function of Colored MNIST.

\end{proof}

\subsection{Validation Accuracy VS Out-of-distribution Accuracy in Colored MNIST}\label{appsec52}

\begin{figure}[ht]
  \subfloat[Validation accuracy and OOD accuracy. They are \emph{negative} correlated, i.e., high validation accuracy leads to low OOD accuracy.\label{fig_VALOOD}]{
	\begin{minipage}[c][1\width]{
	   0.45\textwidth}
	   \centering
	   \includegraphics[width=1.1\textwidth]{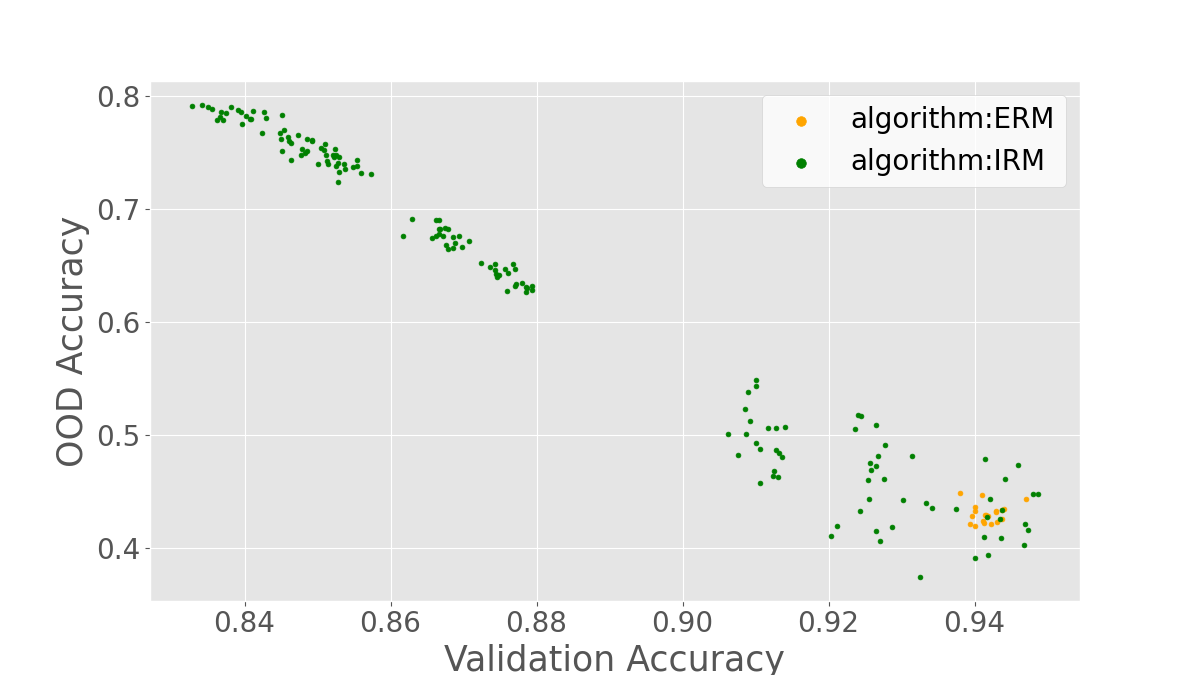}
	   \vspace{-50pt}
	\end{minipage}}
 \hfill 	
  \subfloat[Variation and OOD accuracy. They are \emph{negative} correlated, i.e., low variation leads to high OOD accuracy.\label{fig_DISOOD}]{
	\begin{minipage}[c][1\width]{
	   0.45\textwidth}
	   \centering
	   \includegraphics[width=1.1\textwidth]{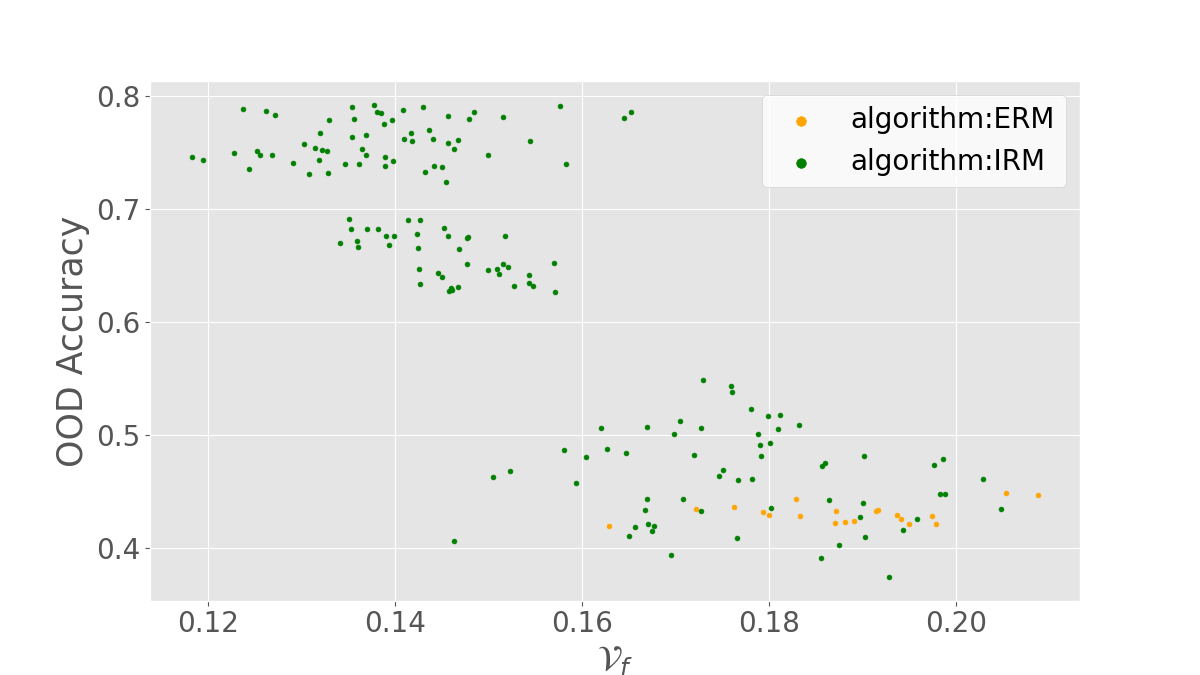}
	   \vspace{-50pt}
	\end{minipage}}
\caption{Experiment Result on Colored MNIST}
\end{figure}

We conduct experiments on the Colored MNIST dataset. As Figure~\ref{fig_VALOOD} shows, validation accuracy on Colored MNIST has a negative relation with OOD accuracy. Therefore, using validation accuracy as a metric to select will result in a poor OOD accuracy. By contrast, the correlation between variation OOD accuracy is also negative, meaning that the smaller the variation is, the higher the OOD accuracy will be.

\section{Appendix: Experiment Details}\label{appsec6}

In this section we list our experiment details. We finish all of our experiment on 8 RTX3090 GPUs and 12 RTX2080 GPUs. It costs over 14,400 GPU hours.

 \paragraph{Architecture \& Dataset}
 We use ResNet50 as our model architecture. The network except last linear and softmax layer is regarded as feature extractor $h(x)$ where the feature dimension $d=2048$. We train our model on three real world OOD datasets (PACS \cite{li2017deeper}, VLCS  \cite{torralba2011unbiased}, OfficeHome \cite{venkateswara2017deep}) by different algorithms and hyperparameters, and collect those models for selection procedure. Both datasets have 4 different environments. For each environment, we split it into 20\% and 80\% splits. The large part is used for training and OOD test. The small part is used for validation.  We compare our criterion with validation criterion on each environment. We use Adam as our optimizer and weight decay is set to zero.  
 
 \paragraph{Data Augmentation}
 Data augmentation is an important method for domain generalization problem. In our experiment, we follow same data augmentation setting in \cite{gulrajani2020search}. We first crops of random size and aspect ratio,
resizing to 224 × 224 pixels, then we do random horizontal flips and color jitter. We also grayscale the image with
10\% probability, and normalize image with the ImageNet channel means and standard deviations.

\paragraph{Hyparameters \& Algorithm}
We search ERM \cite{vapnik1992principles} and four common OOD algorithms (Inter-domain Mixup \cite{yan2020improve}, Group DRO \cite{sagawa2019distributionally}, CORAL \cite{sun2016deep} and IRM \cite{arjovsky2019invariant}). Specific hyper-parameters are listed in Table~\ref{algorithm penalty}. We train each setting for 5 times. 
\begin{table}[t]
    \centering
    \caption{Algorithm specific hyperparmeter choice}
    \begin{tabular}{c| c|c|c| c| c}
    
    \hline\hline
         Algorithms &ERM&CORAL&GroupDRO&Mixup&IRM \\
         \hline
         Penalty&-&$\lambda$=1,0.1,0.01 & $\eta$=0.1,0.01&$\alpha$=0.1,0.2&iter=1000,$\lambda$=1,10\\
         \hline
         lr & \multicolumn{5}{c}{1e-4,5e-5} \\ 
         \hline
         steps & \multicolumn{5}{c}{2500,5000}\\
         \hline\hline
    \end{tabular}
    
    \label{algorithm penalty}
\end{table}

\paragraph{Baseline}
The performance of ``Val'' method is similar to another accuracy-based selection as is shown in \cite{gulrajani2020search}. We compare this method with ours.

\bibliographystyle{alpha}
\bibliography{sample}

\end{document}